\documentclass{article}

\usepackage[nonatbib,preprint]{neurips_2022}

\usepackage[utf8]{inputenc} 
\usepackage[T1]{fontenc}    
\usepackage{hyperref}       
\usepackage{url}            
\usepackage{booktabs}       
\usepackage{amsfonts}       
\usepackage{nicefrac}       
\usepackage{microtype}      
\usepackage{xcolor}         

\usepackage{amsmath}
\usepackage{amsthm}
\usepackage{amssymb}
\usepackage{algorithm}  
\usepackage{algpseudocode}

\newtheorem{prop}{Proposition}
\usepackage{subfig}

\newcommand{\R}{\mathbb{R}}
\newcommand{\sff}{\mathcal}
\newcommand{\norm}[1]{\left\lVert#1\right\rVert}
\newcommand{\ts}{\textsuperscript}  

\usepackage{tcolorbox}  
\newtcolorbox{mybox}{colback  = gray!10!white,sharp corners}  

\usepackage[normalem]{ulem}  

\usepackage{soul}            

\title{ManiFeSt: Manifold-based Feature Selection for Small Data Sets}

\author{%
  David Cohen \\
  Technion - Israel Institute of Technology\\
  \texttt{davidcohenys@gmail.com} 
   \And
   Tal Shnitzer\\
   MIT\\
   \texttt{talsd@mit.edu} \\
  \AND
  Yuval Kluger\\
  Yale University\\
  \texttt{yuval.kluger@yale.edu}\\
  \And
  Ronen Talmon \\
  Technion - Israel Institute of Technology\\
  \texttt{ronen@ee.technion.ac.il} \\
}

\begin{document}

\maketitle

\begin{abstract}
In this paper, we present a new method for few-sample supervised feature selection (FS). 
Our method first learns the manifold of the feature space of each class using kernels capturing multi-feature associations.
Then, based on Riemannian geometry, a composite kernel is computed, extracting the differences between the learned feature associations. 
Finally, a FS score based on spectral analysis is proposed.
Considering multi-feature associations makes our method multivariate by design. 
This in turn allows for the extraction of the hidden manifold underlying the features and avoids overfitting, facilitating few-sample FS.
We showcase the efficacy of our method on illustrative examples and several benchmarks, where our method demonstrates higher accuracy in selecting the informative features compared to competing methods. 
In addition, we show that our FS leads to improved classification and better generalization when applied to test data.
\end{abstract}

\section{Introduction}

Feature selection (FS) plays a vital role in facilitating effective and efficient learning in problems involving high-dimensional data \cite{bolon2022feature,friedman2001elements,duda2006pattern}.
By selecting the relevant features, FS methods, in effect, reduce the dimension of the data, which has been shown useful in improving learning, especially in terms of generalization and noise reduction \cite{remeseiro2019review}.
In contrast to feature extraction (FE), in which the whole feature space is projected into a lower dimension space,
FS eliminates irrelevant and redundant features and preserves interpretability \cite{hira2015review,alelyani2018feature}.

In the literature, there are three main approaches for FS: (i) wrapper, (ii) embedded, and (iii) filter \cite{tang2014feature,alelyani2018feature}.
In the \emph{wrapper approach}, the classifier performance plays an integral role in the feature selection. Concretely, the performance of a specific classifier is maximized by searching for the best subset of features. Since examining all possible subsets is an NP-hard problem, a suboptimal search is often applied.
Still, this approach is considered computationally heavy for problems with large feature spaces \cite{hira2015review,alelyani2018feature}.

The \emph{embedded approach} mitigates the wrapper limitations by incorporating the feature selection in the model training, and thus, avoids multiple optimization processes. Consequently, embedded methods are considered computationally feasible and often preserve the advantages of the wrapper approach. However, both wrapper and embedded methods are prone to overfitting because the selection is part of the training \cite{brown2012conditional,bolon2013review,venkatesh2019review}.

In the \emph{filter approach}, each feature is ranked according to particular criteria independent of the model learning. The various ranking techniques aim to identify the features that best discriminate between the different classes. Then, highly-ranked features are selected and utilized in down stream learning tasks. This approach is computationally efficient and scalable for high-dimensional data. In addition, the classifier performance is not controlled during the FS, mitigating overfitting and enhancing generalization capabilities \cite{jain2018feature}.

In this work, we consider a binary classification problem in a supervised setting, in which the class labels are used in the FS process. 
We propose a FS filter method that identifies the meaningful features by comparing the underlying geometry of the feature spaces of the two classes. 
First, the \emph{feature manifold} of each class is learned using a symmetric positive-definite (SPD) kernel. Then, based on the Riemannian geometry of SPD matrices \cite{pennec2006riemannian,bhatia2009positive}, we build a composite symmetric kernel that captures the differences between the geometries underlying the feature spaces \cite{shnitzer2022spatiotemporal}.
Specifically, we apply spectral analysis to the composite kernel and propose a score that reveals discriminative features. 
We note that our formulation is not limited to the SPD case, and an extension to symmetric positive semi-definite (SPSD) kernels is presented in the appendix.

The proposed method, which we term \emph{ManiFeSt} (Manifold-based Feature Selection), is motivated theoretically and tested on several benchmark datasets. We show empirically that ManiFeSt demonstrates improved capabilities in identifying the informative features compared to other FS filter methods.
The combination of using kernels to capture multi-feature associations and spectral analysis makes ManiFeSt multivariate by design.
We posit that such multivariate information enhances the ability to identify the optimal subset of features, leading to improved performance and generalization capabilities, facilitating FS with only a few labeled samples.
Indeed, we empirically show that ManiFeSt is superior compared to competing univariate methods when the sample size is small.      

Our main contributions are as follows.
(i) We present a new approach for FS from a multivariate standpoint, exploiting the geometry underlying the features. We show that considering multi-feature associations, rather than a univariate perspective based on single features, is useful for identifying the features with high discriminative capabilities.
(ii) We employ a new methodology for feature manifold learning that combines classical manifold learning with the Riemannian geometry of SPD matrices.
(iii) We propose a new algorithm for supervised FS. Our algorithm demonstrates high performance, specifically, improved generalization capabilities, promoting accurate few-sample FS.


\section{Related Work}

Classical filter methods use statistical tests to rank the features. One of the most straightforward methods is based on computing the Pearson's correlation of each feature with the class label \cite{battiti1994using}. The ANOVA F-value \cite{kao2008analysis}, the t-test \cite{davis1986statistics}, and the Fisher score \cite{duda2006pattern} are similarly used for selecting discriminative features. Other scoring techniques, such as information gain (IG) \cite{vergara2014review,ross2014mutual} and Gini-index \cite{shang2007novel}, select features that maximize the purity of each class.

In addition to statistical methods, a fast-growing class of filter methods rely on geometric considerations.
One popular method is the Laplacian score \cite{he2005laplacian,zhao2007spectral,lindenbaum2020differentiable}, which attempts to evaluate the importance of each feature using a graph-Laplacian that is constructed from the samples.
Similarly to the Laplacian score, most of the geometric FS methods consider the geometry underlying the samples. One exception is Relief \cite{kira1992practical} (including its popular extensions \cite{kononenko1994estimating,robnik2003theoretical}), in which the score increases or decreases according to the differences between the values of the feature and its nearest neighbors. 
In contrast to Relief-based methods, 
our method captures the multi-feature associations using kernels, and therefore, it is not limited to nearest neighbors local geometry.

Most existing FS filter methods are univariate, i.e., they consider each feature separately and do not account for multi-feature associations \cite{bolon2015feature,shah2016review,jain2018feature}. Thus, the ability to identify the optimal feature subset may be limited \cite{li2017feature}, leading to degraded performance. 
To mitigate this limitation, mRMR (Minimum Redundancy and Maximum Relevance) \cite{ding2005minimum,zhao2019maximum} and CFS (Correlation-based Feature Selection) \cite{hall1999correlation} algorithms assume that highly correlated features do not contribute to the model and attempt to control feature redundancy. The key idea is to balance between two measures: a relevance measure and a redundancy measure.
While mRMR and CFS algorithms may consider feature associations to avoid selecting highly correlated features, thereby controlling the redundancy in a multivariate manner, the relevance measure is univariate. 
Two notable exceptions are methods that use the trace ratio \cite{nie2008trace} and the generalized fisher score \cite{gu2012generalized} as relevance metrics, which are computed based on a subset of features. However, both methods involve optimization that requires more resources than standard FS filter methods. 

Traditional geometric FS methods such as the Laplacian score \cite{he2005laplacian}, SPEC \cite{zhao2007spectral}, and Relief \cite{robnik2003theoretical} evaluate the importance of the features based on the sample space. In the Laplacian score and SPEC, the constructed kernel reflects the sample associations, and in Relief, the nearest neighbors are determined based on the samples geometry. 
In contrast, our method is applied to the feature space rather than the sample space. Consequently, our method is multivariate, designed to capture complex structures underlying the feature space, leading to improved generalization and consistency. 

\section{ManiFeSt - Manifold-based Feature Selection}
\label{sec:ManiFeSt}

The proposed algorithm for feature selection consists of three stages. First, a feature space representation is constructed for each class using a kernel. Then, we build a composite kernel that is specifically-designed to capture the difference between the classes. Finally, to reveal the significant features, we apply spectral analysis to the composite kernel and propose a FS score.

\subsection{Feature Manifold Learning}
\label{subsec:feature_manifold_learning}

Consider a dataset $\boldsymbol{X}\in \R^{N{\times}d}$ with $N$ samples and $d$ features consisting of two classes.
In order to capture differences in the feature associations between the two classes, we propose to learn the underlying geometry of the feature space of each class using a kernel. For this purpose, according to the class labels, the dataset is divided into two subsets $\boldsymbol{X}^{(1)} = [\boldsymbol{x}_1^{(1)},\cdots,\boldsymbol{x}_d^{(1)}]  \in \R^{N_1{\times}d}$ and $\boldsymbol{X}^{(2)} = [\boldsymbol{x}_1^{(2)},\cdots,\boldsymbol{x}_d^{(2)}] \in \R^{N_2{\times}d}$, where $\boldsymbol{x}_i^{(\ell)} \in \mathbb{R}^{N_{\ell}}$ denotes the $i$th \emph{feature} in the $\ell$th class, $N_1$ and $N_2$ denote the number of samples in the first and second class, respectively, and $N=N_1+N_2$.

For each class $\ell=1,2$, a radial basis function (RBF) kernel $\boldsymbol{K}_{\ell} \in \R^{d \times d}$ is constructed as follows:
\begin{equation}
	\label{eq:kernels}
	\boldsymbol{K}_{\ell}[i,j] = \exp\left(-\norm{\boldsymbol{x}_i^{(\ell)}-\boldsymbol{x}_j^{(\ell)}}^2/2\sigma_\ell^2\right), \ i,j=1,\ldots,d
\end{equation}
where $\sigma_{\ell}$ is a scale factor, typically set to the median of the Euclidean distances up to some scalar.

Using kernels is common practice in nonlinear dimension reduction and manifold learning methods \cite{scholkopf1997kernel,tenenbaum2000global,roweis2000nonlinear,belkin2003laplacian,coifman2006diffusion}. From the standpoint of this approach, the features are viewed as nodes of an undirected weighted graph and the kernel prescribes the weights of the edges connecting the nodes (features). This graph is considered a discrete approximation of the continuous manifold, on which the features reside. Importantly, in contrast to classical manifold learning methods \cite{tenenbaum2000global,roweis2000nonlinear,belkin2003laplacian,coifman2006diffusion}, which typically attempt to learn the manifold underlying the samples, our method learns the manifold underlying the features, capturing information on the feature associations, and thus, making our approach multivariate.
This viewpoint is tightly related to graph signal processing \cite{shuman2013emerging,sandryhaila2013discrete}, where graphs, whose nodes are the features of the signals, are similarly computed. 

One important property of RBF kernels is that they are symmetric positive semi-definite (SPSD) matrices, a fact that we will exploit next. To simplify the exposition, we will assume here that they are strictly positive (SPD), and address the general case of SPSD matrices in Appendix \ref{app:spsd}.
We note that our method is not limited to RBF kernels, and other SPSD kernels could be used instead.

\subsection{Operator Composition on the SPD Manifold}

One way to extract the differences between the feature spaces through their kernel representation is to simply subtract the kernels $\boldsymbol{K}_1-\boldsymbol{K}_2$. Although natural, applying such a linear operation violates the SPD geometry of the kernels and in fact assumes that the kernels live in a linear (vector) space. 
In order to ``respect'' and exploit the underlying Riemannian SPD geometry, we propose to implement the following two-step procedure in a Riemannian manner.
We first find the midpoint $\boldsymbol{M}= (\boldsymbol{K}_1+\boldsymbol{K}_2$)/2, and then, we compute the differences $\boldsymbol{M} -\boldsymbol{K}_1$ or $\boldsymbol{M} -\boldsymbol{K}_2$. 
The Riemannian counterparts of the above Euclidean additions and subtractions are described next.
See Appendix \ref{app:background} for background on the Riemannian geometry of SPD and SPSD matrices. 

Firstly, we compute the mid-point $\boldsymbol{M}$ on the geodesic path connecting $\boldsymbol{K}_1$ and $\boldsymbol{K}_2$ (which coincides with the Riemannian mean \cite{pennec2006riemannian}):
\begin{equation}
	\label{eq:mean}
	\boldsymbol{M} = 	\gamma_{\boldsymbol{K}_1\rightarrow\boldsymbol{K}_2}(1/2) = \boldsymbol{K}_1^{1/2}\left(\boldsymbol{K}_1^{-1/2}\boldsymbol{K}_2\boldsymbol{K}_1^{-1/2}\right)^{1/2}\boldsymbol{K}_1^{1/2}.
\end{equation}
Second, to compute the differences, the two kernels are projected onto the tangent space to the SPD manifold at the mid-point $\boldsymbol{M}$.
By definition, this is given by the logarithmic map of $\boldsymbol{K}_1$ at $\boldsymbol{M}$:
\begin{equation}
	\label{eq:difference}
	\boldsymbol{D} = \text{Log}_{\boldsymbol{M}}(\boldsymbol{K}_1) =
	\boldsymbol{M}^{1/2}\log\left(\boldsymbol{M}^{-1/2}\boldsymbol{K}_1\boldsymbol{M}^{-1/2}\right)\boldsymbol{M}^{1/2}.
\end{equation}
Note that since $\text{Log}_{\boldsymbol{M}}(\boldsymbol{K}_2) = -\text{Log}_{\boldsymbol{M}}(\boldsymbol{K}_1),$
considering only one projection is sufficient.
In Fig.~\ref{fig:operators}, we present an illustration of the construction of the two kernels $\boldsymbol{M}$ and $\boldsymbol{D}$.

\begin{figure}[t]
	\begin{center}
		\includegraphics[width=0.6\textwidth]{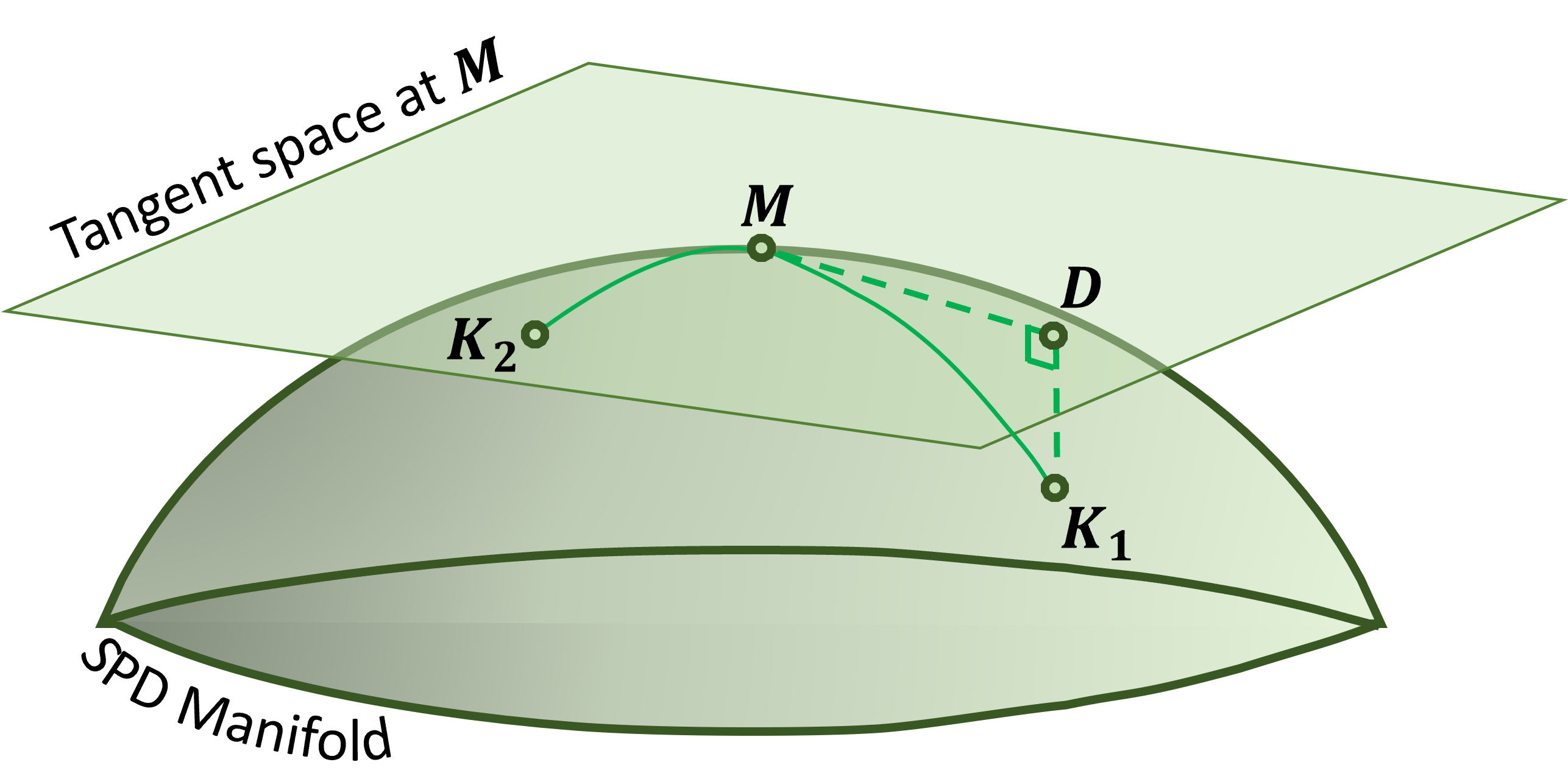}
		\caption{Illustration of the definitions of the operators $\boldsymbol{M}$ and $\boldsymbol{D}$.}
		\label{fig:operators}
	\end{center}
\end{figure}

Background and details on the implementation using SPSD geometry appear in Appendix \ref{app:background} and \ref{app:spsd}, respectively. 

\subsection{Proposed Feature Score}

The proposed FS score relies on the spectral analysis of the composite kernel $\boldsymbol{D}$. 
Let $\lambda_i^{(\boldsymbol{D})}$ and $\boldsymbol{\phi}^{(\boldsymbol{D})}_i \in \R^{d}$ be the eigenvalues and eigenvectors of $\boldsymbol{D}$, respectively.
Note that $\boldsymbol{D}$ is symmetric, and therefore, its eigenvalues are real and its eigenvectors form an orthonormal basis.

The proposed FS score $\boldsymbol{r} \in \mathbb{R}^d$ is given by
\begin{equation}
			\label{eq:score}
			\boldsymbol{r} = \sum_{i=1}^d |\lambda^{(\boldsymbol{D})}_i| \cdot (\boldsymbol{\phi}^{(\boldsymbol{D})}_i \odot \boldsymbol{\phi}^{(\boldsymbol{D})}_i)
\end{equation}
where $\odot$ is the Hadamard (element-wise) product and $\boldsymbol{r}(j)$ is the score of feature $j$.
In words, the magnitude of the eigenvectors is weighted by the eigenvalues and summed over 
to form the ManiFeSt score.
Note that this score is multivariate from two perspectives. First, the kernels capture the associations of each feature with all other features. Second, the eigenvectors $\boldsymbol{\phi}^{(\boldsymbol{D})}_i \in \R^{d}$ integrate kernel entries, thereby incorporating higher-order associations between the features. 

ManiFeSt is summarized in Algorithm \ref{alg:ManiFeSt}.
\begin{algorithm}
	\hspace*{\algorithmicindent} \textbf{Input:}   Two class datasets $\boldsymbol{X}^{(1)}$ and $\boldsymbol{X}^{(2)}$ \\
	\hspace*{\algorithmicindent} \textbf{Output:}   FS score $\boldsymbol{r}$
	
	\caption{ManiFeSt Score}\label{alg:ManiFeSt}
	\begin{algorithmic}[1]
		\State Construct kernels $\boldsymbol{K}_{1}$ and $\boldsymbol{K}_{2}$ for the two datasets    \Comment{According to (\ref{eq:kernels}) }
		\State Build the mean operator $\boldsymbol{M}$ \Comment{According to (\ref{eq:mean}) }
		\State Build the difference operator $\boldsymbol{D}$ \Comment{According to (\ref{eq:difference}) }
        \State Apply eigenvalue decomposition to $\boldsymbol{D}$ and compute the FS score $\boldsymbol{r}$ \Comment{According to \eqref{eq:score}}
        

	\end{algorithmic}
\end{algorithm}

\subsection{Illustrative Example}

We use the MNIST dataset \cite{deng2012mnist} for illustration.
We generate two sets consisting of 1500 images of 4 and 1500 images of 9. In this example, the pixels are viewed as features, and we aim to identify pixels that bear discriminative information on 4 and 9. 
The ManiFeSt results are presented in Fig.~\ref{fig:Ilustration}.


We see in Fig.~\ref{fig:Ilustration}(left) that the two leading eigenvectors of the mid-point kernel, $\boldsymbol{M}$, correspond to the common background and to the common structure of both digits, 4 and 9.
In Fig.~\ref{fig:Ilustration}(middle), we see that the leading eigenvectors of the composite difference kernel, $\boldsymbol{D}$, indeed capture the main conceptual differences between the two digits.
These differences include the gap at the top of the digit 4, the tilt differences in the digits' legs, and the differences between the round upper part of 9 and the square upper part of 4.
As shown in Fig.~\ref{fig:Ilustration}(right), the ManiFeSt score, which weighs the eigenvectors by their respective eigenvalues, provides a consolidated measure of the discriminative pixels. 
In Appendix \ref{app:experiments}, we present an additional illustrative example. 



\begin{figure}[!h]
	\begin{center}
		\includegraphics[width=1\textwidth]{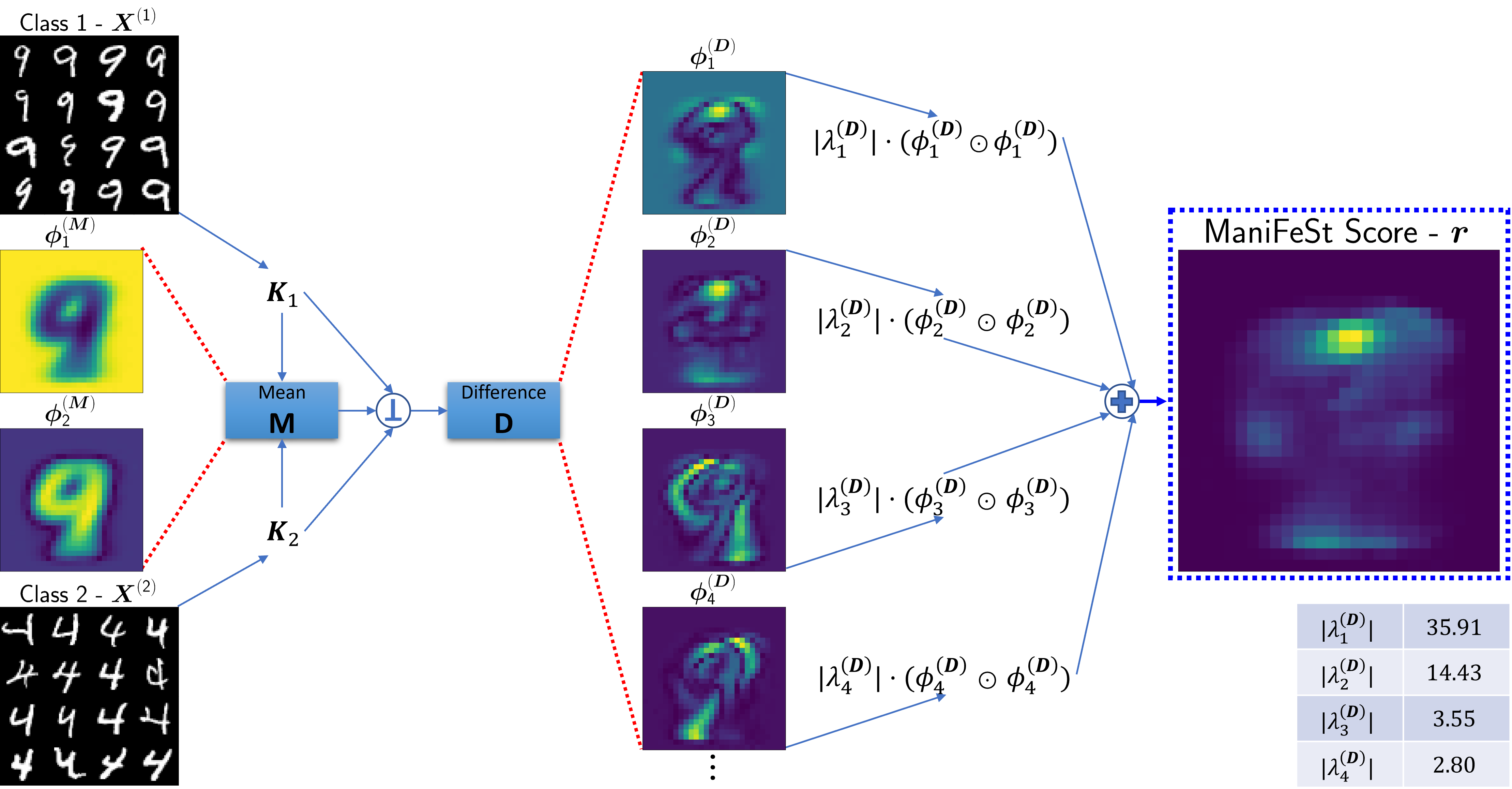}
		\caption{Illustration of the proposed approach and the resulting ManiFeSt score for digit recognition.}
		\label{fig:Ilustration}
	\end{center}
\end{figure}

\section{Theoretical Foundation}
\label{sec:justif}

Our score is related to several previous frameworks that extract new representations (signatures) of data using symmetric positive definite and semi-definite kernels. 
Two notable signatures, defined for shape analysis tasks, are the heat kernel signature \cite{sun2009concise} and the wave kernel signature \cite{aubry2011wave}, both are of the form $\sum_i f(\mu_i)\phi_i^2(x)$, where $\mu_i$ and $\phi_i$ are the eigenvalues and eigenvectors of the Laplace-Beltrami operator and $x$ is a point on the shape.
In another recent work \cite{cheng2020spectral}, such a score was shown to facilitate separation of clustered samples from background samples.
Inspired by these signatures, our score relies on the eigenpairs of the operator $\boldsymbol{D}$, defined in Eq. \eqref{eq:difference} as the Riemannian difference between the kernels representing the feature spaces of the two classes, $\boldsymbol{K_1}$ and $\boldsymbol{K_2}$. 

Each kernel, $\boldsymbol{K}_\ell$, captures intrinsic feature associations that characterize the samples in the class.
The eigenvectors of these kernels can be used as new representations for the feature spaces, extracting intra-class similarities between features. 
For each kernel, $\boldsymbol{K}_\ell$, the most dominant components of these feature associations are captured by eigenvectors that correspond to the largest eigenvalues. 
The motivation for our score then comes from the spectral properties of the difference operator, $\boldsymbol{D}$, which were recently analyzed in \cite{shnitzer2022spatiotemporal}.
This work proved that the leading eigenvectors of $\boldsymbol{D}$ (corresponding to the largest eigenvalues in absolute value) are related to similar eigenvectors of $\boldsymbol{K}_1$ and $\boldsymbol{K}_2$ that correspond to significantly different eigenvalues.
Specifically, it was shown that the eigenvalues of $\boldsymbol{D}$ that correspond to eigenvectors that are (approximately) shared by $\boldsymbol{K}_1$ and $\boldsymbol{K}_2$, are equal to $\lambda^{(\boldsymbol{D})} =\frac{1}{2} \sqrt{\lambda^{(\boldsymbol{K}_1)}\lambda^{(\boldsymbol{K}_2)}}\left(\log(\lambda^{(\boldsymbol{K}_1)})-\log(\lambda^{(\boldsymbol{K}_2)})\right)$.
The term $\sqrt{\lambda^{(\boldsymbol{K}_1)}\lambda^{(\boldsymbol{K}_2)}}$ implies that $\lambda^{(\boldsymbol{D})}$ is dominant only if both $\lambda^{(\boldsymbol{K_1})}$ and $\lambda^{(\boldsymbol{K_2})}$ are dominant. In addition, the term $\left(\log(\lambda^{(\boldsymbol{K}_1)})-\log(\lambda^{(\boldsymbol{K}_2)})\right)$ indicates that $\lambda^{(\boldsymbol{D})}$ is dominant only if $\lambda^{(\boldsymbol{K}_1)}\gg\lambda^{(\boldsymbol{K}_2)}$ or $\lambda^{(\boldsymbol{K}_2)}\gg\lambda^{(\boldsymbol{K}_1)}$.
Therefore, in the context of our work, the operator $\boldsymbol{D}$ emphasizes components representing feature associations that are (i) dominant, and (ii) significantly different in the two classes.
We utilize these properties for feature selection by defining the score $\boldsymbol{r}$ as the sum of the squared eigenvectors, weighted by their corresponding eigenvalues. See Appendix \ref{app:theoretical} for additional theoretical justification.

\section{Experiments}
\label{sec:experiments}

We demonstrate the performance of our method on both synthetic and real datasets and compare it to commonly-used FS methods. 
Throughout the experiments, the data is split to train and test sets with nested cross-validation.
All the competing FS methods are tuned to achieve the best results on the validation set. Additional results, implementation and parameter tuning details are in Appendix \ref{app:experiments}.

\subsection{XOR-100 Problem}

Following \cite{kim2010mlp,bolon2011behavior,yamada2020feature}, we generate a synthetic XOR dataset consisting of $d=100$ binary features and $N=50$ instances. Each feature is sampled from a Bernoulli distribution, and each instance is associated with a label given by $y=f_1 \oplus f_5$, where $\oplus$ is the XOR operation and $f_i$ is the $i$th feature. Thus, only two features, $f_1$ and $f_5$, are relevant for the classification of the label. 

This seemingly simple problem is in fact challenging, especially for existing univariate filter FS methods that consider each feature independently and ignore the inherent feature structure~\cite{bolon2015feature,li2017feature}. 

In Fig. \ref{fig:xor}, we present the normalized feature score obtained by the tested FS methods averaged over 50 Monte-Carlo iterations of data generation. The green circles denote the average score, while the red dashes indicate the standard deviation. In each iteration, the two features with maximal scores are selected, and the average number of correct selections for each method is denoted in parentheses. The results indicate that ManiFeSt perfectly identifies the multivariate behavior of $f_1$ and $f_5$, whereas all other compared methods, except for ReliefF, fail.

We note that this XOR-100 problem is a multivariate problem, because the XOR result depends on $f_1$ and $f_5$.
Nevertheless, we see that ReliefF, which is arguably a univariate method \cite{jovic2015review}, identifies the relevant features. Although ReliefF examines each feature separately, it considers neighboring samples, which evidently provide sufficient multivariate information to correctly detect the features in this example. Still, ManiFeSt, which is multivariate by design, outperforms ReliefF.

\begin{figure}
	\begin{center}
		\includegraphics[width=0.82\textwidth]{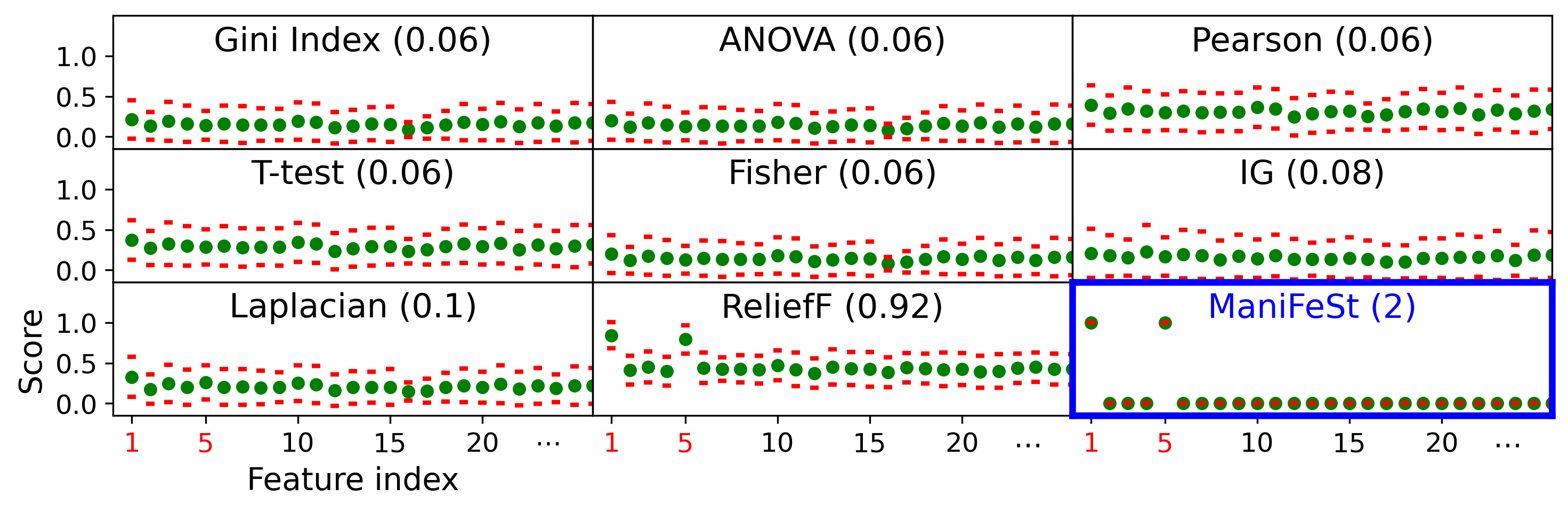}
		\caption{Feature ranking for the XOR-100 problem. Green circles denote the average score, and red lines indicate the standard deviation. The average number of correct FS is denoted in the parentheses.}
		\label{fig:xor}
	\end{center}
\end{figure}

\subsection{Madelon}

We test ManiFeSt on the Madelon synthetic dataset \cite{guyon2008feature} from the NIPS 2003 feature selection challenge. Based on a $5$-dimensional hypercube embedded in $\R^{5}$, the Madelon dataset consists of 2600 points grouped into 32 clusters. Each cluster is normally distributed and centered at one of the hypercube vertices. The clusters are randomly assigned to one of two classes.

Each point is a vector of 500 features, where only 20 are relevant: 5 correspond to the coordinates of the hypercube, and 15 are random linear combinations of them. The remaining features are noise.

The data is divided into train and test sets with a 10-fold cross-validation. 
We consider two cases. In the first case, the FS is based on the whole train set (2340 points). In the second case, the FS is based on only 5 percent of the train set (117 points). Since no ground-truth is available for the relevant features, for evaluation, an SVM classifier is optimized using the entire train set in both cases. 

Fig. \ref{fig:madelon} shows the classification accuracy obtained based on different subsets of features by the tested methods. The curves indicate the average test accuracy, and the shaded area represents the standard deviation. We see based on the red curves that all the methods identify relevant features when the assessment of the FS score is based on the entire train set. Note that selecting too few or too many features may lead to poor classification. The best average test accuracy is $90.73\%$ and is achieved by both ManiFeSt and ReliefF (which obtained the best result in the NIPS 2003 challenge \cite{guyon2007competitive}). 

In addition, when the FS is based on a reduced number of samples, all the methods but ManiFeSt fail to capture the relevant features, as demonstrated by the blue curves. In contrast, the classification obtained by ManiFeSt is unaffected, showing a remarkable robustness to reduction in sample size, thus suggesting good generalization capabilities. 

\begin{figure}
	\begin{center}
		\includegraphics[width=0.82\textwidth]{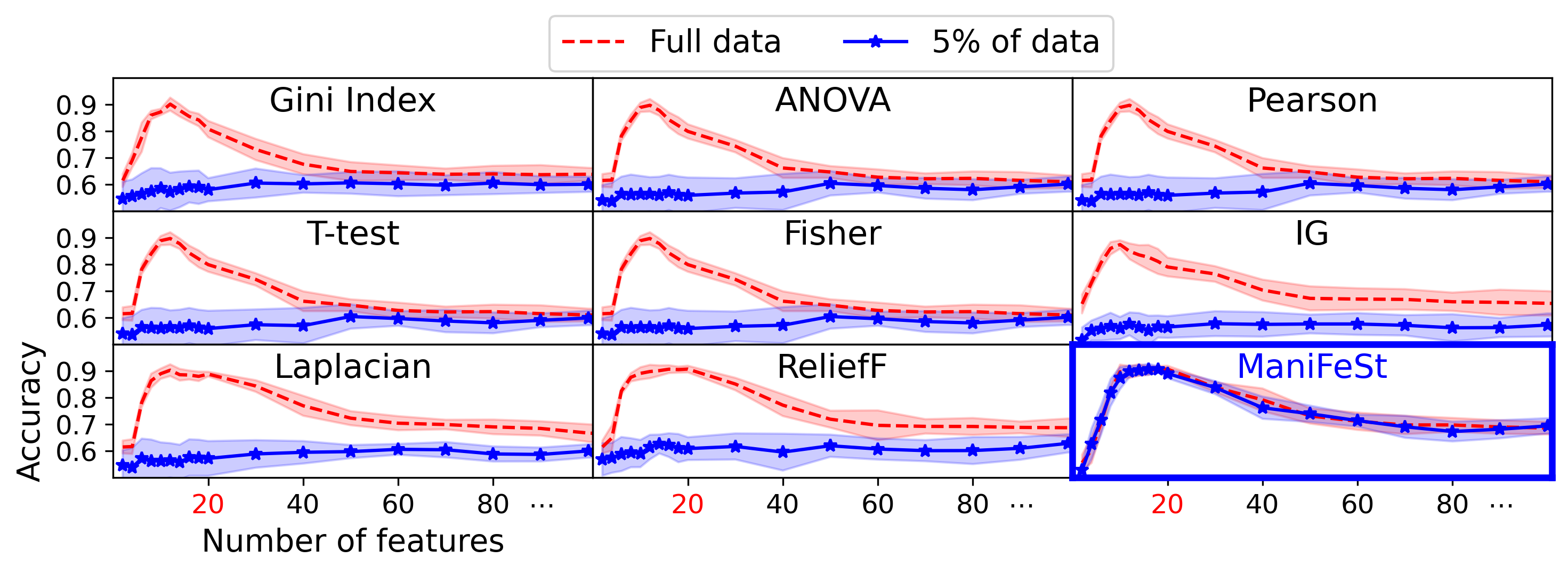}
		\caption{Classification accuracy as a function of the features number on the Madelon dataset. 
		The curves indicate the average test accuracy, and the shaded area represents the standard deviation.}
		\label{fig:madelon}
	\end{center}
\end{figure}


\subsection{Clusters on a Hypercube}

We simulate a variant of the Madelon dataset \cite{guyon2003design} using the scikit-learn function make\_classification(). We consider this variant because here the ground truth is available, whereas the Madelon dataset lacks information on the relevant features. See more details on the dataset generation in Appendix \ref{app:experiments}.

We generate 2000 samples consisting of 200 features with 10 relevant features and split the dataset into train and test sets with 1500 and 500 samples, respectively. For all FS methods, we use only 50 train samples to emphasize the effectiveness of ManiFeSt with only a few labeled samples.
Then, we select the top 10 features according to each FS method. An SVM is optimized using the entire train set with the selected features.
We repeat this procedure using 50 cross-validation iterations.

Fig. \ref{fig:hypercube}(a) presents the number of correct selections obtained by the tested FS methods. The median and average values are denoted by red lines and circles. The boundaries of the box indicate the 25\ts{th} and 75\ts{th} percentiles. We see that ManiFeSt outperforms the competing methods by a large margin using a relatively small number of samples.

Fig. \ref{fig:hypercube}(b) shows the t-SNE visualization \cite{van2008visualizing} of the test samples using all the features (left), top 10 features selected by ReliefF (middle), and top 10 features selected by ManiFeSt (right). The color (red and green) denotes the (hidden) class label. We see that both FS methods lead to a better class separation and that the separation using ManiFeSt is more pronounced. We report that ManiFeSt yields 10 (out of 10) correct selections, whereas ReliefF only 6. In addition, the average classification accuracy is depicted in parentheses, further demonstrating the advantage of ManiFeSt over ReliefF.



\begin{figure}
	\begin{center}
		\includegraphics[width=0.82\textwidth]{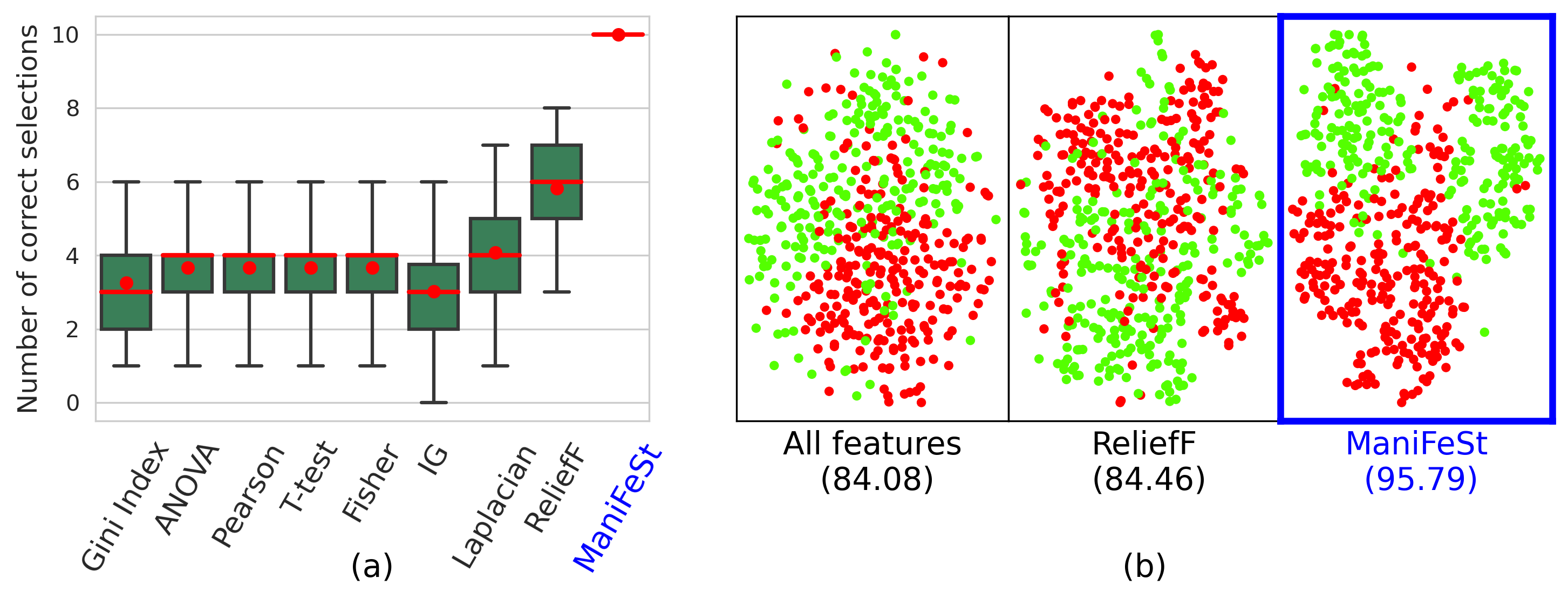}
		\caption{Results on the simulated variant of the Madelon dataset. (a) Boxplot of the number of correct selections indicating the 25\ts{th} and the 75\ts{th} percentiles. The median and average are denoted by red lines and circles, respectively. (b) t-SNE visualization of the test samples using all the features (left), top 10 features selected by ReliefF (middle), and top 10 features selected by ManiFeSt (right).}
		\label{fig:hypercube}
	\end{center}
\end{figure}

\subsection{Colon Cancer Gene Expression}

We test a dataset of colon cancer gene expression samples \cite{alon1999broad}, which is relatively small, typical to the biological domain. 
The dataset consists of the expression levels of 2000 genes (features) in 62 tissues (samples), of which 22 are normal, and 40 are of colon cancer.
The samples are split to 90\% train and 10\% test sets.
Results are averaged over 50 cross-validation iterations.

Fig.~\ref{fig:colon}(a) shows the average classification accuracy for different subsets of features. The solid and dashed curves represent the average validation and test accuracy, respectively. 
ManiFeSt achieves the best test accuracy of $86.9\%$, whereas the test accuracy of competing methods is $85.8\%$ or below. 

Even though generalization can be improved by applying FS, the FS itself is still prone to overfitting.
Indeed, from the gap between the validation and test accuracy, we see that ManiFeSt generalizes well compared to all other competing methods. 
To further test the generalization capabilities of ManiFeSt, in Appendix \ref{app:experiments}, we present the generalization error obtained by ManiFeSt for three different kernel scales, i.e., $\sigma_\ell$ in Eq. \eqref{eq:kernels}. The results imply that the larger the scale is, i.e., the more feature associations are captured by the kernel, the smaller the generalization error becomes.
This may suggest that considering the associations between the features enhances generalization, in contrast to the competing methods that only consider univariate feature properties.

While ManiFeSt exhibits enhanced generalization capabilities, its maximal performance is achieved by using more features (200 compared to 40). Our empirical examination revealed that ManiFeSt selects some irrelevant features because it analyzes feature associations rather than each feature separately.
Therefore, ManiFeSt might identify features without any discriminative capabilities, through their connections to other relevant and discriminative features (see details in Appendix \ref{app:experiments}). Still, despite the selection of irrelevant features, ManiFeSt facilitates the best test accuracy.

To alleviate the selection of irrelevant features, we propose combining classical univariate criteria with our multivariate score of ManiFeSt \eqref{eq:score}. In Fig.~\ref{fig:colon}(b), we present the results obtained when combining ManiFeSt with ReliefF by summing their normalized feature scores. We see that this simple combination results in an improved performance. Now, the maximal accuracy is $88.1\%$, and it is obtained by selecting only 40 features. This result calls for further research, exploring systematic ways to combine the multivariate standpoint of ManiFeSt with univariate considerations.

The enhanced generalization capabilities are demonstrated here only with respect to \emph{filter} methods, since embedded and wrapper methods typically suffer from large generalization errors when applied to small datasets \cite{brown2012conditional,bolon2013review,venkatesh2019review}.
To support this claim, we report that a recent embedded method applied to the colon dataset obtained test accuracy of $83.85\%$ \cite{https://doi.org/10.48550/arxiv.2106.06468}, outperforming various other embedded methods. By using the same train-test split scheme (49/13) as in \cite{https://doi.org/10.48550/arxiv.2106.06468}, ManiFeSt achieves test accuracy of $85.38\%$ with 400 features. The combination with ReliefF obtains $85.23\%$ accuracy with 80 features. 
See more comparisons in Appendix \ref{app:experiments}.  
We note that filter methods are usually used as preprocessing for wrapper and embedded methods \cite{alshamlan2015genetic,shaban2020new,peng2010novel}. In such an approach, ManiFeSt may provide explainable prepossessing without eliminating multivariate structures unlike existing filters.




\begin{figure}
	\begin{center}
		\includegraphics[width=1\textwidth]{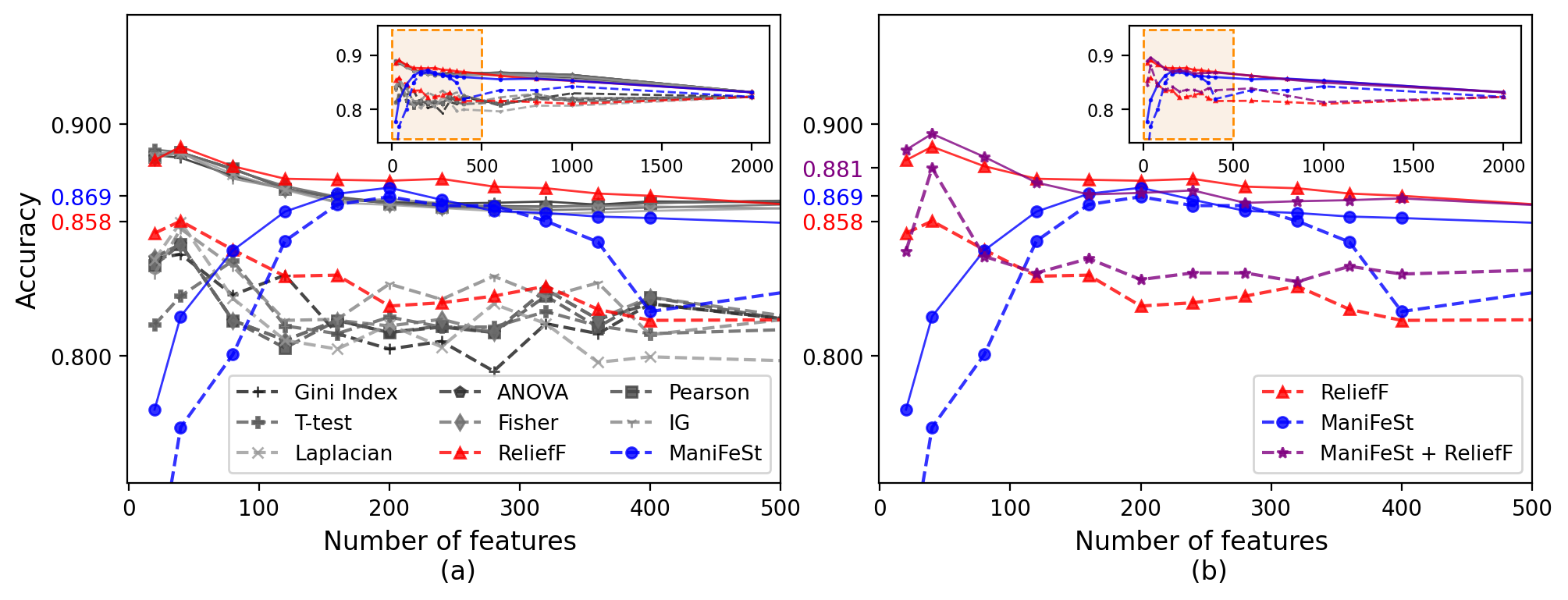}
		\caption{Accuracy as a function of the number of features on the gene expression dataset. 
		The dashed and solid curves represent the average test and validation accuracy, respectively. 
		}
		\label{fig:colon}
	\end{center}
\end{figure}

\section{Limitations and Future Directions}
\label{sec:limitations}

Our method has several limitations; we outline them and propose possible remedies.
First, ManiFeSt is designed for binary classification. The general multi-class case could be addressed by one-vs-all or one-vs-one approaches \cite{izetta2017improved}. Alternatively, extending the Riemannian composition to multiple kernels is possible and will be addressed in future work. 
Second, as we demonstrated empirically, analyzing multivariate associations rather than univariate, sometimes leads to the selection of irrelevant features. In future work, we will investigate combinations of (classical) univariate criteria and ManiFeSt, making systematic and precise the presented ad hoc combination of ReliefF and ManiFeSt.
In addition, a similar approach can be used for reducing feature redundancies, which our method currently does not account for.
Third, as a kernel method, ManiFeSt cannot be applied to very large feature spaces (of order of magnitude $>10$K). Usually, this limitation is addressed by computing sparse kernels, however, in our work it is not possible because the composite kernels are not necessarily sparse, even if the two kernels are sparse. 
A possible remedy in such cases is to use invertible dimension reduction algorithms, such as PCA, prior to the application of ManiFeSt, select the features in the low dimensional space, and then map them back to the original space.

\section{Conclusions}
\label{sec:conclusions}

In this work, we propose a supervised FS method. The proposed method, which is termed ManiFeSt, identifies discriminative features by comparing the multi-feature associations of each class. 
To this end, ManiFeSt employs a geometric approach that combines manifold learning and Riemannian geometry.
In contrast to common FS filter methods, our method learns the geometry in the feature space underlying the multi-feature associations rather than applying a univariate analysis.
We demonstrate that our multivariate approach reveals various data structures and facilitates improved generalization and consistency for FS based on small datasets, outperforming competing FS methods.

\section*{Acknowledgement}
This work was supported by the European Union's Horizon 2020 research and innovation programme under grant agreement No. 802735-ERC-DIFFOP.



\bibliographystyle{abbrv}
\bibliography{ManiFeSt_Manifold_based_Feature_Selection_for_Small_Data_Sets}

\begin{thebibliography}{10}

\bibitem{alelyani2018feature}
S.~Alelyani, J.~Tang, and H.~Liu.
\newblock Feature selection for clustering: A review.
\newblock {\em Data Clustering}, pages 29--60, 2018.

\bibitem{alon1999broad}
U.~Alon, N.~Barkai, D.~A. Notterman, K.~Gish, S.~Ybarra, D.~Mack, and A.~J.
  Levine.
\newblock Broad patterns of gene expression revealed by clustering analysis of
  tumor and normal colon tissues probed by oligonucleotide arrays.
\newblock {\em Proceedings of the National Academy of Sciences},
  96(12):6745--6750, 1999.

\bibitem{alshamlan2015genetic}
H.~M. Alshamlan, G.~H. Badr, and Y.~A. Alohali.
\newblock Genetic bee colony ({GBC}) algorithm: A new gene selection method for
  microarray cancer classification.
\newblock {\em Computational biology and chemistry}, 56:49--60, 2015.

\bibitem{arsigny2007geometric}
V.~Arsigny, P.~Fillard, X.~Pennec, and N.~Ayache.
\newblock Geometric means in a novel vector space structure on symmetric
  positive-definite matrices.
\newblock {\em SIAM journal on matrix analysis and applications},
  29(1):328--347, 2007.

\bibitem{atashgahi2019brain}
Z.~Atashgahi, J.~Pieterse, S.~Liu, D.~C. Mocanu, R.~Veldhuis, and
  M.~Pechenizkiy.
\newblock A brain-inspired algorithm for training highly sparse neural
  networks.
\newblock {\em arXiv preprint arXiv:1903.07138}, 2019.

\bibitem{aubry2011wave}
M.~Aubry, U.~Schlickewei, and D.~Cremers.
\newblock The wave kernel signature: A quantum mechanical approach to shape
  analysis.
\newblock In {\em 2011 IEEE international conference on computer vision
  workshops (ICCV workshops)}, pages 1626--1633.

\bibitem{battiti1994using}
R.~Battiti.
\newblock Using mutual information for selecting features in supervised neural
  net learning.
\newblock {\em IEEE Transactions on neural networks}, 5(4):537--550, 1994.

\bibitem{belkin2003laplacian}
M.~Belkin and P.~Niyogi.
\newblock Laplacian eigenmaps for dimensionality reduction and data
  representation.
\newblock {\em Neural computation}, 15(6):1373--1396, 2003.

\bibitem{bhatia2009positive}
R.~Bhatia.
\newblock Positive definite matrices.
\newblock In {\em Positive Definite Matrices}. Princeton university press,
  2009.

\bibitem{bhatia2019bures}
R.~Bhatia, T.~Jain, and Y.~Lim.
\newblock On the bures--wasserstein distance between positive definite
  matrices.
\newblock {\em Expositiones Mathematicae}, 37(2):165--191, 2019.

\bibitem{bolon2022feature}
V.~Bol{\'o}n-Canedo, A.~Alonso-Betanzos, L.~Mor{\'a}n-Fern{\'a}ndez, and
  B.~Cancela.
\newblock Feature selection: From the past to the future.
\newblock In {\em Advances in Selected Artificial Intelligence Areas}, pages
  11--34. Springer, 2022.

\bibitem{bolon2011behavior}
V.~Bol{\'o}n-Canedo, N.~S{\'a}nchez-Marono, and A.~Alonso-Betanzos.
\newblock On the behavior of feature selection methods dealing with noise and
  relevance over synthetic scenarios.
\newblock In {\em The 2011 International Joint Conference on Neural Networks},
  pages 1530--1537. IEEE.

\bibitem{bolon2013review}
V.~Bol{\'o}n-Canedo, N.~S{\'a}nchez-Maro{\~n}o, and A.~Alonso-Betanzos.
\newblock A review of feature selection methods on synthetic data.
\newblock {\em Knowledge and information systems}, 34(3):483--519, 2013.

\bibitem{bolon2015feature}
V.~Bol{\'o}n-Canedo, N.~S{\'a}nchez-Maro{\~n}o, and A.~Alonso-Betanzos.
\newblock {\em Feature selection for high-dimensional data}.
\newblock Springer, 2015.

\bibitem{bonnabel2010riemannian}
S.~Bonnabel and R.~Sepulchre.
\newblock Riemannian metric and geometric mean for positive semidefinite
  matrices of fixed rank.
\newblock {\em SIAM Journal on Matrix Analysis and Applications},
  31(3):1055--1070, 2010.

\bibitem{brown2012conditional}
G.~Brown, A.~Pocock, M.-J. Zhao, and M.~Luj{\'a}n.
\newblock Conditional likelihood maximisation: a unifying framework for
  information theoretic feature selection.
\newblock {\em The journal of machine learning research}, 13(1):27--66, 2012.

\bibitem{cheng2020spectral}
X.~Cheng and G.~Mishne.
\newblock Spectral embedding norm: Looking deep into the spectrum of the graph
  {Laplacian}.
\newblock {\em SIAM journal on imaging sciences}, 13(2):1015--1048, 2020.

\bibitem{coifman2006diffusion}
R.~R. Coifman and S.~Lafon.
\newblock Diffusion maps.
\newblock {\em Applied and computational harmonic analysis}, 21(1):5--30, 2006.

\bibitem{davis1986statistics}
J.~C. Davis and R.~J. Sampson.
\newblock {\em Statistics and data analysis in geology}, volume 646.
\newblock Wiley New York, 1986.

\bibitem{deng2012mnist}
L.~Deng.
\newblock The {MNIST} database of handwritten digit images for machine learning
  research.
\newblock {\em IEEE Signal Processing Magazine}, 29(6):141--142, 2012.

\bibitem{ding2005minimum}
C.~Ding and H.~Peng.
\newblock Minimum redundancy feature selection from microarray gene expression
  data.
\newblock {\em Journal of bioinformatics and computational biology},
  3(02):185--205, 2005.

\bibitem{duda2006pattern}
R.~O. Duda, P.~E. Hart, et~al.
\newblock {\em Pattern classification}.
\newblock John Wiley \& Sons, 2006.

\bibitem{fawzi2021faster}
H.~Fawzi and H.~Goulbourne.
\newblock Faster proximal algorithms for matrix optimization using jacobi-based
  eigenvalue methods.
\newblock {\em Advances in Neural Information Processing Systems}, 34, 2021.

\bibitem{friedman2001elements}
J.~Friedman, T.~Hastie, R.~Tibshirani, et~al.
\newblock {\em The elements of statistical learning}, volume~1.
\newblock Springer series in statistics New York, 2001.

\bibitem{gu2012generalized}
Q.~Gu, Z.~Li, and J.~Han.
\newblock Generalized fisher score for feature selection.
\newblock {\em arXiv preprint arXiv:1202.3725}, 2012.

\bibitem{guyon2003design}
I.~Guyon.
\newblock Design of experiments of the {NIPS} 2003 variable selection
  benchmark.
\newblock In {\em NIPS 2003 workshop on feature extraction and feature
  selection}, volume 253, page~40.

\bibitem{guyon2008feature}
I.~Guyon, S.~Gunn, M.~Nikravesh, and L.~A. Zadeh.
\newblock {\em Feature extraction: foundations and applications}, volume 207.
\newblock Springer, 2008.

\bibitem{guyon2007competitive}
I.~Guyon, J.~Li, T.~Mader, P.~A. Pletscher, G.~Schneider, and M.~Uhr.
\newblock Competitive baseline methods set new standards for the {NIPS} 2003
  feature selection benchmark.
\newblock {\em Pattern recognition letters}, 28(12):1438--1444, 2007.

\bibitem{hall1999correlation}
M.~A. Hall et~al.
\newblock Correlation-based feature selection for machine learning.
\newblock 1999.

\bibitem{he2005laplacian}
X.~He, D.~Cai, and P.~Niyogi.
\newblock Laplacian score for feature selection.
\newblock {\em Advances in neural information processing systems}, 18, 2005.

\bibitem{hira2015review}
Z.~M. Hira and D.~F. Gillies.
\newblock A review of feature selection and feature extraction methods applied
  on microarray data.
\newblock {\em Advances in bioinformatics}, 2015.

\bibitem{hsu2003practical}
C.-W. Hsu, C.-C. Chang, C.-J. Lin, et~al.
\newblock A practical guide to support vector classification, 2003.

\bibitem{izetta2017improved}
J.~Izetta, P.~F. Verdes, and P.~M. Granitto.
\newblock Improved multiclass feature selection via list combination.
\newblock {\em Expert Systems with Applications}, 88:205--216, 2017.

\bibitem{jain2018feature}
D.~Jain and V.~Singh.
\newblock Feature selection and classification systems for chronic disease
  prediction: A review.
\newblock {\em Egyptian Informatics Journal}, 19(3):179--189, 2018.

\bibitem{jovic2015review}
A.~Jovi{\'c}, K.~Brki{\'c}, and N.~Bogunovi{\'c}.
\newblock A review of feature selection methods with applications.
\newblock In {\em 2015 38th international convention on information and
  communication technology, electronics and microelectronics (MIPRO)}, pages
  1200--1205. IEEE, 2015.

\bibitem{kao2008analysis}
L.~S. Kao and C.~E. Green.
\newblock Analysis of variance: is there a difference in means and what does it
  mean?
\newblock {\em Journal of Surgical Research}, 144(1):158--170, 2008.

\bibitem{kim2010mlp}
G.~Kim, Y.~Kim, H.~Lim, and H.~Kim.
\newblock An {MLP}-based feature subset selection for {HIV}-1 protease cleavage
  site analysis.
\newblock {\em Artificial intelligence in medicine}, 48(2-3):83--89, 2010.

\bibitem{kira1992practical}
K.~Kira and L.~A. Rendell.
\newblock A practical approach to feature selection.
\newblock In {\em Machine learning proceedings 1992}, pages 249--256. Elsevier.

\bibitem{kononenko1994estimating}
I.~Kononenko.
\newblock Estimating attributes: Analysis and extensions of {RELIEF}.
\newblock In {\em European conference on machine learning}, pages 171--182.
  Springer, 1994.

\bibitem{landa2021doubly}
B.~Landa, R.~R. Coifman, and Y.~Kluger.
\newblock Doubly stochastic normalization of the gaussian kernel is robust to
  heteroskedastic noise.
\newblock {\em SIAM journal on mathematics of data science}, 3(1):388--413,
  2021.

\bibitem{lazar2013batch}
C.~Lazar, S.~Meganck, J.~Taminau, D.~Steenhoff, A.~Coletta, C.~Molter, D.~Y.
  Weiss-Sol{\'\i}s, R.~Duque, H.~Bersini, and A.~Now{\'e}.
\newblock Batch effect removal methods for microarray gene expression data
  integration: a survey.
\newblock {\em Briefings in bioinformatics}, 14(4):469--490, 2013.

\bibitem{li2017feature}
J.~Li, K.~Cheng, S.~Wang, F.~Morstatter, R.~P. Trevino, J.~Tang, and H.~Liu.
\newblock Feature selection: A data perspective.
\newblock {\em ACM computing surveys (CSUR)}, 50(6):1--45, 2017.

\bibitem{li2018feature}
J.~Li, K.~Cheng, S.~Wang, F.~Morstatter, R.~P. Trevino, J.~Tang, and H.~Liu.
\newblock Feature selection: A data perspective.
\newblock {\em ACM Computing Surveys (CSUR)}, 50(6):94, 2018.

\bibitem{lin2019riemannian}
Z.~Lin.
\newblock Riemannian geometry of symmetric positive definite matrices via
  cholesky decomposition.
\newblock {\em SIAM Journal on Matrix Analysis and Applications},
  40(4):1353--1370, 2019.

\bibitem{lindenbaum2020differentiable}
O.~Lindenbaum, U.~Shaham, J.~Svirsky, E.~Peterfreund, and Y.~Kluger.
\newblock Differentiable unsupervised feature selection based on a gated
  {Laplacian}.
\newblock {\em arXiv preprint arXiv:2007.04728}, 2020.

\bibitem{malago2018wasserstein}
L.~Malago, L.~Montrucchio, and G.~Pistone.
\newblock Wasserstein riemannian geometry of gaussian densities.
\newblock {\em Information Geometry}, 1(2):137--179, 2018.

\bibitem{massart2020quotient}
E.~Massart and P.-A. Absil.
\newblock Quotient geometry with simple geodesics for the manifold of
  fixed-rank positive-semidefinite matrices.
\newblock {\em SIAM Journal on Matrix Analysis and Applications},
  41(1):171--198, 2020.

\bibitem{moakher2005differential}
M.~Moakher.
\newblock A differential geometric approach to the geometric mean of symmetric
  positive-definite matrices.
\newblock {\em SIAM Journal on Matrix Analysis and Applications},
  26(3):735--747, 2005.

\bibitem{nie2008trace}
F.~Nie, S.~Xiang, Y.~Jia, C.~Zhang, and S.~Yan.
\newblock Trace ratio criterion for feature selection.
\newblock In {\em AAAI}, volume~2, pages 671--676, 2008.

\bibitem{peng2010novel}
Y.~Peng, Z.~Wu, and J.~Jiang.
\newblock A novel feature selection approach for biomedical data
  classification.
\newblock {\em Journal of Biomedical Informatics}, 43(1):15--23, 2010.

\bibitem{pennec2006riemannian}
X.~Pennec, P.~Fillard, and N.~Ayache.
\newblock A {Riemannian} framework for tensor computing.
\newblock {\em International Journal of computer vision}, 66(1):41--66, 2006.

\bibitem{remeseiro2019review}
B.~Remeseiro and V.~Bolon-Canedo.
\newblock A review of feature selection methods in medical applications.
\newblock {\em Computers in biology and medicine}, 112:103375, 2019.

\bibitem{robnik2003theoretical}
M.~Robnik-{\v{S}}ikonja and I.~Kononenko.
\newblock Theoretical and empirical analysis of {ReliefF} and {RReliefF}.
\newblock {\em Machine learning}, 53(1):23--69, 2003.

\bibitem{ross2014mutual}
B.~C. Ross.
\newblock Mutual information between discrete and continuous data sets.
\newblock {\em PloS one}, 9(2):e87357, 2014.

\bibitem{roweis2000nonlinear}
S.~T. Roweis and L.~K. Saul.
\newblock Nonlinear dimensionality reduction by locally linear embedding.
\newblock {\em science}, 290(5500):2323--2326, 2000.

\bibitem{sandryhaila2013discrete}
A.~Sandryhaila and J.~M. Moura.
\newblock Discrete signal processing on graphs.
\newblock {\em IEEE transactions on signal processing}, 61(7):1644--1656, 2013.

\bibitem{scholkopf1997kernel}
B.~Sch{\"o}lkopf, A.~Smola, and K.-R. M{\"u}ller.
\newblock Kernel principal component analysis.
\newblock In {\em International conference on artificial neural networks},
  pages 583--588. Springer, 1997.

\bibitem{shaban2020new}
W.~M. Shaban, A.~H. Rabie, A.~I. Saleh, and M.~Abo-Elsoud.
\newblock A new {COVID}-19 patients detection strategy ({CPDS}) based on hybrid
  feature selection and enhanced knn classifier.
\newblock {\em Knowledge-Based Systems}, 205:106270, 2020.

\bibitem{shah2016review}
F.~P. Shah and V.~Patel.
\newblock A review on feature selection and feature extraction for text
  classification.
\newblock In {\em 2016 international conference on wireless communications,
  signal processing and networking (WiSPNET)}, pages 2264--2268. IEEE, 2016.

\bibitem{shang2007novel}
W.~Shang, H.~Huang, H.~Zhu, Y.~Lin, Y.~Qu, and Z.~Wang.
\newblock A novel feature selection algorithm for text categorization.
\newblock {\em Expert Systems with Applications}, 33(1):1--5, 2007.

\bibitem{shnitzer2022spatiotemporal}
T.~Shnitzer, H.-T. Wu, and R.~Talmon.
\newblock Spatiotemporal analysis using {Riemannian} composition of diffusion
  operators.
\newblock {\em arXiv preprint arXiv:2201.08530}, 2022.

\bibitem{shuman2013emerging}
D.~I. Shuman, S.~K. Narang, P.~Frossard, A.~Ortega, and P.~Vandergheynst.
\newblock The emerging field of signal processing on graphs: Extending
  high-dimensional data analysis to networks and other irregular domains.
\newblock {\em IEEE signal processing magazine}, 30(3):83--98, 2013.

\bibitem{singh2002gene}
D.~Singh, P.~G. Febbo, K.~Ross, D.~G. Jackson, J.~Manola, C.~Ladd, P.~Tamayo,
  A.~A. Renshaw, A.~V. D'Amico, J.~P. Richie, et~al.
\newblock Gene expression correlates of clinical prostate cancer behavior.
\newblock {\em Cancer cell}, 1(2):203--209, 2002.

\bibitem{sun2009concise}
J.~Sun, M.~Ovsjanikov, and L.~Guibas.
\newblock A concise and provably informative multi-scale signature based on
  heat diffusion.
\newblock In {\em Computer graphics forum}, volume~28, pages 1383--1392. Wiley
  Online Library, 2009.

\bibitem{tang2014feature}
J.~Tang, S.~Alelyani, and H.~Liu.
\newblock Feature selection for classification: A review.
\newblock {\em Data classification: Algorithms and applications}, page~37,
  2014.

\bibitem{tenenbaum2000global}
J.~B. Tenenbaum, V.~d. Silva, and J.~C. Langford.
\newblock A global geometric framework for nonlinear dimensionality reduction.
\newblock {\em science}, 290(5500):2319--2323, 2000.

\bibitem{van2008visualizing}
L.~Van~der Maaten and G.~Hinton.
\newblock Visualizing data using {t-SNE}.
\newblock {\em Journal of machine learning research}, 9(11), 2008.

\bibitem{vandereycken2009embedded}
B.~Vandereycken, P.-A. Absil, and S.~Vandewalle.
\newblock Embedded geometry of the set of symmetric positive semidefinite
  matrices of fixed rank.
\newblock In {\em 2009 IEEE/SP 15th Workshop on Statistical Signal Processing},
  pages 389--392. IEEE, 2009.

\bibitem{vandereycken2013riemannian}
B.~Vandereycken, P.-A. Absil, and S.~Vandewalle.
\newblock A riemannian geometry with complete geodesics for the set of positive
  semidefinite matrices of fixed rank.
\newblock {\em IMA Journal of Numerical Analysis}, 33(2):481--514, 2013.

\bibitem{venkatesh2019review}
B.~Venkatesh and J.~Anuradha.
\newblock A review of feature selection and its methods.
\newblock {\em Cybernetics and information technologies}, 19(1):3--26, 2019.

\bibitem{vergara2014review}
J.~R. Vergara and P.~A. Est{\'e}vez.
\newblock A review of feature selection methods based on mutual information.
\newblock {\em Neural computing and applications}, 24(1):175--186, 2014.

\bibitem{xiao2017fashion}
H.~Xiao, K.~Rasul, and R.~Vollgraf.
\newblock Fashion-mnist: a novel image dataset for benchmarking machine
  learning algorithms.
\newblock {\em arXiv preprint arXiv:1708.07747}, 2017.

\bibitem{yamada2020feature}
Y.~Yamada, O.~Lindenbaum, S.~Negahban, and Y.~Kluger.
\newblock Feature selection using stochastic gates.
\newblock In {\em International Conference on Machine Learning}, pages
  10648--10659. PMLR, 2020.

\bibitem{https://doi.org/10.48550/arxiv.2106.06468}
J.~Yang, O.~Lindenbaum, and Y.~Kluger.
\newblock Locally sparse neural networks for tabular biomedical data, 2021.

\bibitem{zass2005unifying}
R.~Zass and A.~Shashua.
\newblock A unifying approach to hard and probabilistic clustering.
\newblock In {\em Tenth IEEE International Conference on Computer Vision
  (ICCV'05) Volume 1}, volume~1, pages 294--301. IEEE, 2005.

\bibitem{zhao2019maximum}
Z.~Zhao, R.~Anand, and M.~Wang.
\newblock Maximum relevance and minimum redundancy feature selection methods
  for a marketing machine learning platform.
\newblock In {\em 2019 IEEE International Conference on Data Science and
  Advanced Analytics (DSAA)}, pages 442--452.

\bibitem{zhao2007spectral}
Z.~Zhao and H.~Liu.
\newblock Spectral feature selection for supervised and unsupervised learning.
\newblock In {\em Proceedings of the 24th international conference on Machine
  learning}, pages 1151--1157, 2007.

\end{thebibliography}



\newpage


\clearpage

\appendix

\section{Background on Riemannian Geometry}
\label{app:background}

\subsection{Riemannian Manifold of SPD Matrices}\label{appsb:SPDbg}
Let $\sff{S}_d$ denote the set of the symmetric matrices in $\R^{d{\times}d}$. $\boldsymbol{K}\in\sff{S}_d$ is an SPD matrix if all its eigenvalues are strictly positive. We denote the set of $d{\times}d$ SPD matrices as $\sff{P}_d$. The tangent space at $\boldsymbol{K}\in\sff{P}_d$ is the space of symmetric matrices and is denoted by $\sff{T}_{\boldsymbol{K}}\sff{P}_d$. 
When the tangent space is endowed with a proper metric the space of SPD matrices forms 
a differential Riemannian manifold \cite{moakher2005differential}. 
Various metrics have been proposed in the literature \cite{arsigny2007geometric,bhatia2019bures,lin2019riemannian,malago2018wasserstein,pennec2006riemannian}, of which the affine invariant metric \cite{pennec2006riemannian} and the log-Euclidean metric \cite{arsigny2007geometric} are arguably the most widely used, both allow formal definitions of geometric notions such as the geodesic path on the manifold.
We focus here on the affine invariant metric, defined for $\boldsymbol{S}_1,\boldsymbol{S}_2\in\sff{T}_{\boldsymbol{K}}\sff{P}_d$ as follows:
\begin{equation}
    \left\langle\boldsymbol{S}_1,\boldsymbol{S}_2\right\rangle_{\boldsymbol{K}} = \left\langle\boldsymbol{K}^{-1/2}\boldsymbol{S}_1\boldsymbol{K}^{-1/2},\boldsymbol{K}^{-1/2}\boldsymbol{S}_2\boldsymbol{K}^{-1/2}\right\rangle \label{app:eq:spd_inner}
\end{equation}
where $\left\langle\boldsymbol{A}_1,\boldsymbol{A}_2\right\rangle=\mathrm{Tr}\left(\boldsymbol{A}_1^T\boldsymbol{A}_2\right)$ is the standard Euclidean inner product.
Based on this metric, the unique geodesic path on the SPD manifold connecting two matrices $\boldsymbol{K}_1,\boldsymbol{K}_2\in\sff{P}_d$ is given by:
\begin{equation}
	\gamma^{\sff{P}}_{\boldsymbol{K}_1\rightarrow\boldsymbol{K}_2}(t) = \boldsymbol{K}_1^{1/2}\left(\boldsymbol{K}_1^{-1/2}\boldsymbol{K}_2\boldsymbol{K}_1^{-1/2}\right)^t\boldsymbol{K}_1^{1/2}\label{app:eq:spd_gp}
\end{equation}
where $0 \le t \le 1$. It holds that $\gamma^{\sff{P}}_{\boldsymbol{K}_1\rightarrow\boldsymbol{K}_2}(0)=\boldsymbol{K}_1$ and $\gamma^{\sff{P}}_{\boldsymbol{K}_1\rightarrow\boldsymbol{K}_2}(1)=\boldsymbol{K}_2$.

The projection of a point (symmetric matrix) in the tangent space $\boldsymbol{S} \in \sff{T}_{\boldsymbol{K}}\sff{P}_d$ to the SPD manifold is given by the following exponential map:
\begin{equation}
	\text{Exp}_{\boldsymbol{K}}(\boldsymbol{S}) =
	\boldsymbol{K}^{1/2}\exp\left(\boldsymbol{K}^{-1/2}\boldsymbol{S}\boldsymbol{K}^{-1/2}\right)\boldsymbol{K}^{1/2}
\end{equation}
where the result $\widetilde{\boldsymbol{K}}=\text{Exp}_{\boldsymbol{K}}(\boldsymbol{S}) \in \sff{P}_d$ is an SPD matrix.

The inverse projection of $\widetilde{\boldsymbol{K}} \in \sff{P}_d$ to the tangent space is given by the following logarithmic map:
\begin{equation}
	\text{Log}_{\boldsymbol{K}}(\widetilde{\boldsymbol{K}}) =
	\boldsymbol{K}^{1/2}\log\left(\boldsymbol{K}^{-1/2}\widetilde{\boldsymbol{K}}\boldsymbol{K}^{-1/2}\right)\boldsymbol{K}^{1/2}\label{eq:logmap}
\end{equation}
where the result $\text{Log}_{\boldsymbol{K}}(\boldsymbol{\widetilde{K}}) \in \sff{T}_{\boldsymbol{K}}\sff{P}_d$ is a symmetric matrix in the tangent space.

Further details on the SPD manifold are provided in \cite{bhatia2009positive,pennec2006riemannian}.

\subsection{Riemannian Manifold of SPSD Matrices}
To mitigate the requirement for full rank SPD matrices, several Riemannian geometries have been proposed for symmetric positive semi-definite matrices (SPSD) \cite{bonnabel2010riemannian,massart2020quotient,vandereycken2009embedded,vandereycken2013riemannian}. 
We focus on the one proposed in \cite{bonnabel2010riemannian}, which generalizes the affine-invariant geometry (presented in Appendix \ref{appsb:SPDbg}), forming the basis of our method. 
This SPSD geometry coincides with the SPD affine-invariant metric, when restricted to SPD matrices.

Let $\sff{S}^+_{d,k}$ denote the set of SPSD matrices of size $d\times d$ and fixed rank $k<d$.
Any $\boldsymbol{K}\in\sff{S}^{+}_{d,k}$ can be represented by $\boldsymbol{K}=\boldsymbol{G}\boldsymbol{P}\boldsymbol{G}^T$, where $\boldsymbol{P}\in\sff{P}_k$ is a $k\times k$ SPD matrix, $\boldsymbol{G}\in\sff{V}_{d,k}$, and $\sff{V}_{d,k}$ denotes the set of $d\times k$ matrices with orthonormal columns. 
This representation of $\boldsymbol{K}$ can be obtained by its eigenvalue decomposition for example.
This representation implies that SPSD matrices can be represented by the pair $\left(\boldsymbol{G},\boldsymbol{P}\right)$, which is termed the structure space representation. Note that the structure space representation is unique up to orthogonal transformations, $\boldsymbol{O}\in\sff{O}_k$, i.e., $\boldsymbol{K}\cong\left(\boldsymbol{G}\boldsymbol{O},\boldsymbol{O}^T\boldsymbol{P}\boldsymbol{O}\right)$. 
It follows that the space $\sff{S}^+_{d,k}$ has a quotient manifold representation, $\sff{S}^+_{d,k}\cong\left(\sff{V}_{d,k}\times\sff{P}_k\right)/\sff{O}_k$.
The structure space representation pair is thus composed of SPD matrices $\boldsymbol{P}\in\sff{P}_k$, whose space forms a Riemannian manifold with the affine-invariant metric \eqref{app:eq:spd_inner}, and matrices $\boldsymbol{G}\in\sff{G}_{d,k}$, where $\sff{G}_{d,k}$ denotes the set of $k$-dimensional subspaces of $\mathbb{R}^d$.
The set $\sff{G}_{d,k}$ forms the Grassmann manifold with an appropriate inner product on its tangent space $\sff{T}_{\boldsymbol{G}}\sff{G}_{d,k}=\{\boldsymbol{\Delta}=\boldsymbol{G}_{\perp}\boldsymbol{B}\ \vert\ \boldsymbol{B}\in\mathbb{R}^{(d-k)\times k}\}$, given by $\left\langle\boldsymbol{\Delta}_1,\boldsymbol{\Delta}_2\right\rangle_{\boldsymbol{G}}=\left\langle\boldsymbol{B}_1,\boldsymbol{B}_2\right\rangle$, where $\boldsymbol{G}_{\perp}\in\sff{V}_{d,d-k}$ is the orthogonal complement of $\boldsymbol{G}$.
To define the geodesic path between two points, $\boldsymbol{G}_1$ and $\boldsymbol{G}_2$, on the Grassmann manifold, let $\boldsymbol{G}_2^T\boldsymbol{G}_1=\boldsymbol{O}_2\boldsymbol{\Sigma}\boldsymbol{O}_1^T$ denote the singular value decomposition (SVD), where $\boldsymbol{O}_1,\boldsymbol{O}_2\in\mathbb{R}^{k\times k}$, $\boldsymbol{\Sigma}$ is a diagonal matrix with $\sigma_i=\cos\theta_i$ on its diagonal, and $\theta_i$ denote the principal angles between the two subspaces represented by $\boldsymbol{G}_1$ and $\boldsymbol{G}_2$. Assuming $\max_i \theta_i\leq\pi/2$, the closed-form for the geodesic path is then given by:
\begin{equation}
    \gamma^{\sff{G}}_{\boldsymbol{G}_1\rightarrow\boldsymbol{G}_2}(t)=\boldsymbol{G}_1\boldsymbol{O}_1\cos\left(\boldsymbol{\Theta}t\right)+\boldsymbol{X}\sin\left(\boldsymbol{\Theta}t\right)\label{app:eq:grassmann_gp}
\end{equation}
where $\boldsymbol{\Theta}=\mathrm{diag}(\theta_1,\dots,\theta_k)$, $\theta_i=\arccos\sigma_i$, and $\boldsymbol{X}=\left(\boldsymbol{I}-\boldsymbol{G}_1\boldsymbol{G}_1^T\right)\boldsymbol{G}_2\boldsymbol{O}_2\left(\sin\boldsymbol{\Theta}\right)^{\dagger}$, where $()^{\dagger}$ denotes the pseudo-inverse. 

Following the structure space representation of $\sff{S}^+_{d,k}$, its tangent space is defined in \cite{bonnabel2010riemannian} by $\sff{T}_{\left(\boldsymbol{G},\boldsymbol{P}\right)}\sff{S}^+_{d,k}=\{(\boldsymbol{\Delta},\boldsymbol{S}):\boldsymbol{\Delta}\in\sff{T}_{\boldsymbol{G}}\sff{G}_{d,k},\boldsymbol{S}\in\sff{T}_{\boldsymbol{P}}\sff{P}_k\}$, and the inner product on the tangent space is given by the sum of the inner products on the two components:
\begin{equation}
    \left\langle (\boldsymbol{\Delta}_1,\boldsymbol{S}_1),(\boldsymbol{\Delta}_2,\boldsymbol{S}_2) \right\rangle_{(\boldsymbol{G},\boldsymbol{P})} = \left\langle\boldsymbol{\Delta}_1,\boldsymbol{\Delta}_2\right\rangle_{\boldsymbol{G}} + m\left\langle\boldsymbol{S}_1,\boldsymbol{S}_2\right\rangle_{\boldsymbol{P}}\label{app:eq:spsd_inner}
\end{equation}
for $m>0$, where $(\boldsymbol{\Delta}_\ell,\boldsymbol{S}_\ell)\in\sff{T}_{(\boldsymbol{G},\boldsymbol{P})}\sff{S}^+_{d,k}$, and $\left\langle\boldsymbol{S}_1,\boldsymbol{S}_2\right\rangle_{\boldsymbol{P}}$ is defined as in \eqref{app:eq:spd_inner}.
There is no closed-form expression for the geodesic path connecting two points on $\sff{S}^+_{d,k}$, however, the following approximation is proposed in \cite{bonnabel2010riemannian}:
\begin{equation}
    \tilde{\gamma}_{\boldsymbol{K}_1\rightarrow\boldsymbol{K}_2}(t) = \gamma^{\sff{G}}_{\boldsymbol{G}_1\rightarrow\boldsymbol{G}_2}(t)\gamma^{\sff{P}}_{\boldsymbol{P}_1\rightarrow\boldsymbol{P}_2}(t)(\gamma^{\sff{G}}_{\boldsymbol{G}_1\rightarrow\boldsymbol{G}_2}(t))^T\label{app:eq:spsd_gp}
\end{equation}
where $\boldsymbol{K}_{\ell}\cong\left(\boldsymbol{G}_{\ell},\boldsymbol{P}_{\ell}\right)$, $\boldsymbol{K}_{\ell}\in\sff{S}^+_{d,k}$, $\boldsymbol{P}_{\ell}=\boldsymbol{O}_{\ell}^T\boldsymbol{G}_{\ell}^T\boldsymbol{K}_{\ell}\boldsymbol{G}_{\ell}\boldsymbol{O}_{\ell}$ due to the non-uniqueness of the decomposition (up to orthogonal transformations), $\gamma^{\sff{P}}_{\boldsymbol{P}_1\rightarrow\boldsymbol{P}_2}(t)$ is defined by \eqref{app:eq:spd_gp} and $\gamma^{\sff{G}}_{\boldsymbol{G}_1\rightarrow\boldsymbol{G}_2}(t)$ is defined by \eqref{app:eq:grassmann_gp}.

\section{Difference Operator for SPSD Matrices}
\label{app:spsd}

Following \cite{shnitzer2022spatiotemporal}, and based on the approximation of the geodesic path in \eqref{app:eq:spsd_gp}, we define the mean and difference operators for two SPSD matrices, $\boldsymbol{K}_1$ and $\boldsymbol{K}_2$, whose structure space representation is given by $\boldsymbol{K}_{\ell}\cong\left(\boldsymbol{G}_{\ell},\boldsymbol{P}_{\ell}\right)$ where $\boldsymbol{G}_{\ell} \in\sff{V}_{d,k}$ and $\boldsymbol{P}_{\ell} \in \mathcal{P}_k$. 
Define $\boldsymbol{G}_2^T\boldsymbol{G}_1=\boldsymbol{O}_2\boldsymbol{\Sigma}\boldsymbol{O}_1^T$ as the SVD of $\boldsymbol{G}_2^T\boldsymbol{G}_1$ and set $\boldsymbol{P}_{\ell}=\boldsymbol{O}_{\ell}^T\boldsymbol{G}_{\ell}^T\boldsymbol{K}_{\ell}\boldsymbol{G}_{\ell}\boldsymbol{O}_{\ell}$, $\ell=1,2$.
The mean operator is then defined analogously to \eqref{eq:mean} as the mid-point of $\tilde{\gamma}_{\boldsymbol{K}_1\rightarrow\boldsymbol{K}_2}(t)$:
\begin{equation}
    \widetilde{\boldsymbol{M}}=\tilde{\gamma}_{\boldsymbol{K}_1\rightarrow\boldsymbol{K}_2}(0.5)=\gamma^{\sff{G}}_{\boldsymbol{G}_1\rightarrow\boldsymbol{G}_2}(0.5)\gamma^{\sff{P}}_{\boldsymbol{P}_1\rightarrow\boldsymbol{P}_2}(0.5)(\gamma^{\sff{G}}_{\boldsymbol{G}_1\rightarrow\boldsymbol{G}_2}(0.5))^T\label{app:eq:spsd_mean}
\end{equation}
Denote the structure space representation of the mean operator by $\widetilde{\boldsymbol{M}}\cong (\boldsymbol{G}_{\boldsymbol{M}},\boldsymbol{P}_{\boldsymbol{M}})$ and define $\boldsymbol{G}_1^T\boldsymbol{G}_{\boldsymbol{M}}=\widetilde{\boldsymbol{O}}_{1}\widetilde{\boldsymbol{\Sigma}}\boldsymbol{O}_{\boldsymbol{M}}^T$ as the SVD of $\boldsymbol{G}_1^T\boldsymbol{G}_{\boldsymbol{M}}$.
Set $\boldsymbol{P}_{\boldsymbol{M}}=\boldsymbol{O}_{\boldsymbol{M}}^T\boldsymbol{G}_{\boldsymbol{M}}^T\widetilde{\boldsymbol{M}}\boldsymbol{G}_{\boldsymbol{M}}\boldsymbol{O}_{\boldsymbol{M}}$, and $\widetilde{\boldsymbol{P}}_{1}=\widetilde{\boldsymbol{O}}_1^T\boldsymbol{G}_1^T\boldsymbol{K}_1\boldsymbol{G}_1\widetilde{\boldsymbol{O}}_1$.
The SPSD difference operator is defined by:
\begin{equation}
    \widetilde{\boldsymbol{D}} = \gamma^{\sff{G}}_{\boldsymbol{G}_{\boldsymbol{M}}\rightarrow\boldsymbol{G}_1}(1)\mathrm{Log}_{\boldsymbol{P}_{\boldsymbol{M}}}(\widetilde{\boldsymbol{P}}_1)(\gamma^{\sff{G}}_{\boldsymbol{G}_{\boldsymbol{M}}\rightarrow\boldsymbol{G}_1}(1))^T\label{app:eq:spsd_diff}
\end{equation}
where the logarithmic map on the SPD manifold is defined in \eqref{eq:logmap} and the geodesic path on the Grassmann manifold is defined in \eqref{app:eq:grassmann_gp}.

The computation of $\widetilde{\boldsymbol{M}}$ and $\widetilde{\boldsymbol{D}}$, as well as the resulting ManiFeSt score for the SPSD case, are summarized in Algorithm \ref{app:alg:ManiFeSt_spsd}.

\begin{algorithm}
	\hspace*{\algorithmicindent} \textbf{Input:}   Two class datasets $\boldsymbol{X}^{(1)}$ and $\boldsymbol{X}^{(2)}$ \\
	\hspace*{\algorithmicindent} \textbf{Output:}   FS score $\boldsymbol{r}$ (SPSD case)
	
	\caption{ManiFeSt Score for SPSD Matrices}\label{app:alg:ManiFeSt_spsd}
	\begin{algorithmic}[1]
		\State Construct kernels $\boldsymbol{K}_{1}$ and $\boldsymbol{K}_{2}$ for the two datasets    \Comment{According to (\ref{eq:kernels}) }
		\State Set $k=\min\{\mathrm{rank}(\boldsymbol{K}_1),\mathrm{rank}(\boldsymbol{K}_2)\}$
		\State Define $\boldsymbol{K}_{\ell}=\boldsymbol{G}_{\ell}\boldsymbol{P}_{\ell}\boldsymbol{G}_{\ell}$ and $\boldsymbol{P}_{\ell}=\boldsymbol{O}_{\ell}^T\boldsymbol{G}_{\ell}^T\boldsymbol{K}_{\ell}\boldsymbol{G}_{\ell}\boldsymbol{O}_{\ell}$ \newline where $\boldsymbol{G}_{2}^T\boldsymbol{G}_{1}=\boldsymbol{O}_2\boldsymbol{\Sigma}\boldsymbol{O}_1$
		\State Compute $\gamma^{\sff{G}}_{\boldsymbol{G}_1\rightarrow\boldsymbol{G}_2}(0.5)$ \Comment{According to (\ref{app:eq:grassmann_gp}) }
		\State Compute $\gamma^{\sff{P}}_{\boldsymbol{P}_1\rightarrow\boldsymbol{P}_2}(0.5)$ \Comment{According to (\ref{app:eq:spd_gp}) }
		\State Build the mean operator $\widetilde{\boldsymbol{M}}$ \Comment{According to (\ref{app:eq:spsd_mean}) }
		\State Define $\widetilde{\boldsymbol{M}}=\boldsymbol{G}_{\boldsymbol{M}}\boldsymbol{P}_{\boldsymbol{M}}\boldsymbol{G}_{\boldsymbol{M}}$, $\boldsymbol{P}_{\boldsymbol{M}}=\boldsymbol{O}_{\boldsymbol{M}}^T\boldsymbol{G}_{\boldsymbol{M}}^T\boldsymbol{M}\boldsymbol{G}_{\boldsymbol{M}}\boldsymbol{O}_{\boldsymbol{M}}$ and $\widetilde{\boldsymbol{P}}_{1}=\widetilde{\boldsymbol{O}}_{1}^T\boldsymbol{G}_{1}^T\boldsymbol{K}_1\boldsymbol{G}_{1}\widetilde{\boldsymbol{O}}_{1}$ \newline where $\boldsymbol{G}_{1}^T\boldsymbol{G}_{\boldsymbol{M}}=\widetilde{\boldsymbol{O}}_{1}\widetilde{\boldsymbol{\Sigma}}\boldsymbol{O}_{\boldsymbol{M}}$
		\State Compute $\mathrm{Log}_{\boldsymbol{P}_{\boldsymbol{M}}}(\widetilde{\boldsymbol{P}}_1)$ \Comment{According to (\ref{eq:logmap}) }
		\State Compute $\gamma^{\sff{G}}_{\boldsymbol{G}_{\boldsymbol{M}}\rightarrow\boldsymbol{G}_1}(1)$ \Comment{According to (\ref{app:eq:grassmann_gp}) }
		\State Build the difference operator $\widetilde{\boldsymbol{D}}$ \Comment{According to (\ref{app:eq:spsd_diff}) }
        \State Apply eigenvalue decomposition to $\widetilde{\boldsymbol{D}}$ and compute the FS score $\boldsymbol{r}$ \Comment{According to \eqref{eq:score}}
        
	\end{algorithmic}
\end{algorithm}

\section{Experiments -- More Details and Additional Results}
\label{app:experiments}

\subsection{Implementation Details}

In all the experiments, the data is split to train and test sets with nested cross-validation. The data normalization, FS, and SVM hyper-parameters tuning are applied to the train set to prevent test-train leakage.

All FS methods are tuned to achieve the best accuracy results on the validation set or the maximal number of correct selections when the ground-truth relevant feature identities are available. The train set is divided using 10-fold cross-validation for all datasets. This procedure is repeated with shuffled samples for small datasets for better tuning.

\paragraph{Data normalization.}
Following \cite{atashgahi2019brain}, the features in the Madelon dataset are normalized by removing the mean and rescaling to unit variance. This is implemented using the standard sklearn function. 
The other datasets do not require normalization. 

\paragraph{FS hyperparameter tuning.}

Both IG and ReliefF FS methods have a number of nearest neighbors parameter, since the extension of the classic IG score from discrete to continuous features makes use of the $k$ nearest neighbors \cite{ross2014mutual}. 
We tune the number of neighbors for IG and ReliefF over the grid $k=\{1,3,5,10,15,20,30,50,100\}$. For the Laplacian score, the samples' kernel scale is tuned to the $i_{th}$ percentile of Euclidean distances over the grid $i=\{1,5,10,30,50,70,90,95,99\}$. 
ManiFeSt only requires tuning of the features' kernel scale. For the illustrative example, the scale $\sigma_{\ell}$ is set to the median of Euclidean distances, for best visualization. For the XOR and Madelon problems, the scale is set to the median of Euclidean distances multiplied by a factor $0.1$. 
Since the multi-feature associations of the relevant features are distinct in these two datasets, no scale tuning was required. Conversely, for the remaining datasets, the scale factor is tuned to the $i_{th}$ percentile of the Euclidean distances over the grid $i=\{5,10,30,50,70,90,95\}$. 

\paragraph{SVM hyperparameter tuning.}
When the ground-truth of which features are relevant is not available, we apply an SVM classifier to the selected subset of features in order to evaluate the FS. 
For the SVM hyperparameter tuning, we follow \cite{hsu2003practical}. We use an RBF kernel and perform a grid search on the penalty parameter $C$ and the kernel scale $\gamma$. $C$ and $\gamma$ are tuned over exponentially growing sequences, $C=\{2^{-5},2^{-2},2^{1},2^{4},2^{7},2^{10},2^{13}\}$ and $\gamma=\{2^{-15},2^{-12},2^{-9},2^{-6},2^{-3},2^{0},2^{3}\}$.

\paragraph{Computing resources.}
All the experiments were performed using Python on a standard PC with an Intel i7-7700k CPU and 64GB of RAM without GPUs.
We note that according to a recent work \cite{fawzi2021faster}, using GPUs could allow a faster computation of the eigenvalue decomposition required by ManiFeSt.

\paragraph{FS source code.} 
The competing methods were implemented as follows.
The IG~\cite{vergara2014review} and ANOVA~\cite{kao2008analysis} methods were computed using the scikit-learn package. For Gini-index~\cite{shang2007novel}, t-test~\cite{davis1986statistics}, Fisher~\cite{duda2006pattern}, Laplacian~\cite{he2005laplacian}, and ReliefF~\cite{robnik2003theoretical}, we use the skfeature repository developed by the Arizona State University~\cite{li2018feature}. The Pearson correlation~\cite{battiti1994using} is implemented by the built-in Panda package correlation. 


\paragraph{Details on the hypercube dataset.}

We create a $10$-dimensional hypercube embedded in $\R^{10}$. Then, 2000 points are generated and grouped into 4 clusters. The data in each cluster are normally distributed and centered at one of vertices of the hypercube. We define two classes, where each class  consists of two clusters. The partition of the 4 clusters to the two classed is performed in an arbitrarily manner.

To create the dataset, these 10-dimensional points are mapped to $\mathbb{R}^{200}$ by appending coordinates with random noise, so that each point in the dataset consists of 200 features out of which only 10 are relevant.
To obtain the ManiFeSt score results in Figure \ref{fig:hypercube}, the kernel matrix $\boldsymbol{K}$ was symmetrically normalized in three iterations according to $\boldsymbol{K}_{\ell+1}=\boldsymbol{D}^{-1/2}_\ell\boldsymbol{K}_\ell\boldsymbol{D}^{-1/2}_\ell$, where $\boldsymbol{D}_\ell$ is a diagonal matrix with $d_i=\sum_j K(i,j)$ on its diagonal. This normalization converges in infinity to a doubly stochastic kernel, which was shown to be more robust to noise \cite{landa2021doubly,zass2005unifying}.
When using the unnormalized kernel from \eqref{eq:kernels}, ManiFeSt obtains slightly lower results (median $9$, mean $9.2$ and 25\ts{th}-75\ts{th} percentiles $[9,10]$) but still outperforms all the competing methods.
We remark that the results of ManiFeSt in all other applications were obtained with the unnormalized kernel.



\subsection{Additional Results}

\subsubsection{Fashion-MNIST: another illustrative example}
\label{app:Ilustration2}

We use the Fashion-MNIST dataset \cite{xiao2017fashion} for illustration.
We generate two sets: one consists of 1500 images of pants and the other consists of 1500 images of shirts. 
Fig.~\ref{fig:Ilustration2} is the same as Fig.~\ref{fig:Ilustration}, presenting the results obtained by ManiFeSt for this example.


%
In Fig.~\ref{fig:Ilustration2}(middle), we see that the leading eigenvectors of the composite difference kernel, $\boldsymbol{D}$, indeed capture the main conceptual differences between the two clothes.
These differences include the gap between the pants' legs, the gap between the shirts' sleeves, and the shirt collar.
As shown in Fig.~\ref{fig:Ilustration2}(right), the ManiFeSt score, which weighs the eigenvectors by their respective eigenvalues, provides a consolidated measure of the discriminative pixels. 

\begin{figure}[!h]
	\begin{center}
		\includegraphics[width=1\textwidth]{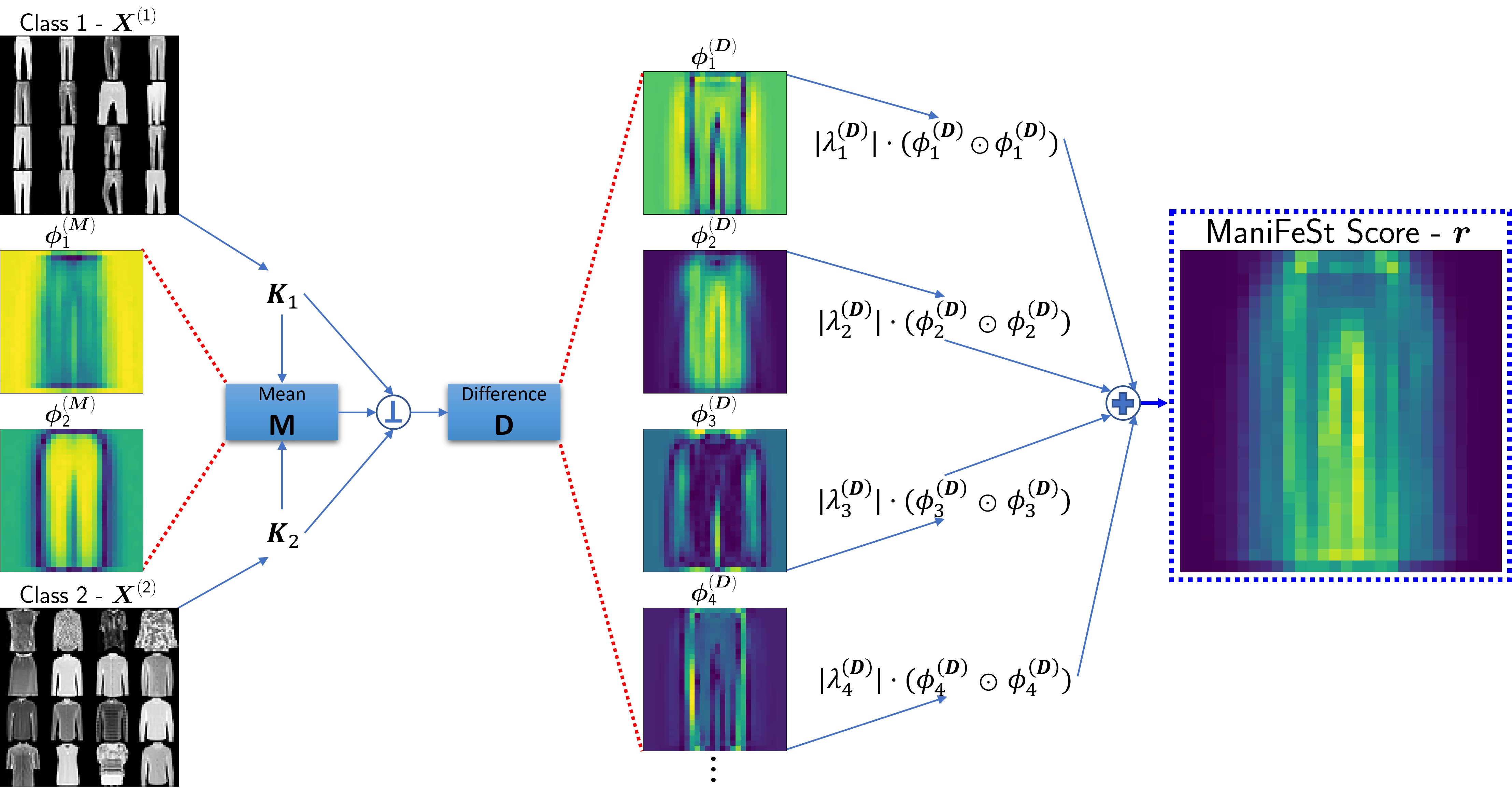}
		\caption{Ilustration of the proposed scheme and the resulting ManiFeSt score for images from the Fashion-MNIST dataset.}
		\label{fig:Ilustration2}
	\end{center}
\end{figure}




\subsubsection{Colon cancer gene expression: additional results}

\paragraph{More generalization tests.} 
To further demonstrate the generalization capabilities of ManiFeSt, we examine the effect of the kernel scale, i.e., $\sigma_\ell$ in Eq. \eqref{eq:kernels}, on the results of the colon dataset. In Fig.~\ref{fig:colon_app}, we present the generalization error obtained by ManiFeSt for three different kernel scales. We see that the larger the scale is, i.e., the more feature associations are captured by the kernel, the smaller the generalization error becomes. This result indicates that the multi-feature associations taken into account by ManiFeSt play a central role in its favorable generalization capabilities.




\begin{figure}
	\begin{center}
		\includegraphics[width=0.6\textwidth]{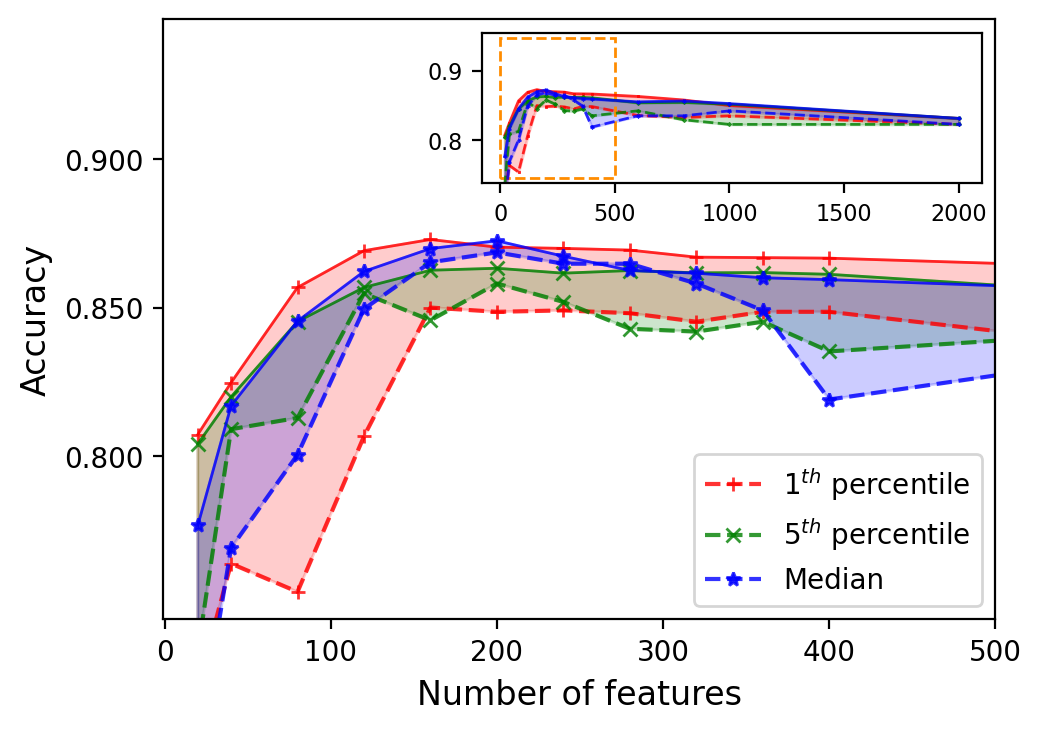}
		\caption{ManiFeSt with three different kernel scales. The dashed and solid lines represent the average test and validation accuracy, and the shaded area represents the validation-test generalization error.}
		\label{fig:colon_app}
	\end{center}
\end{figure}

We conclude this section on generalization with a possible direction for future investigation.
We speculate that the large generalization error demonstrated in Fig.~\ref{fig:colon} by the competing methods may suggest the presence of significant batch effects in the colon dataset, which is prototypical to such biological data \cite{lazar2013batch}. In contrast, the smaller generalization error achieved by ManiFeSt could indicate its robustness to batch effects.
Therefore, the robustness of ManiFeSt to batch effects will be studied in future work.

\paragraph{Comparison with embedded methods.}
To complement the experimental study, we report here recent results on the colon dataset \cite{https://doi.org/10.48550/arxiv.2106.06468} obtained by various embedded methods. 
These results are displayed in Table \ref{tab:colon_comparison} along with our result obtained by ManiFeSt. For a fair comparison, in this experiment we use the same train-test split (49/13) and average the results over 50 cross-validation iterations. 
We see in the table that ManiFeSt achieves the best mean accuracy, but with a larger standard deviation, compared with the leading competing embedded method (LLSPIN proposed in  \cite{https://doi.org/10.48550/arxiv.2106.06468}).


\begin{table}
  \caption{ManiFeSt Comparison with Embedded Methods}
  \label{tab:colon_comparison}
  \centering
  \begin{tabular}{ll}
    \toprule
    \cmidrule(r){1-2}
    Method          & Accuracy $\pm$ STD \\
    \midrule
    LASSO             & $81.54 \pm 9.85$    \\
    SVC               & $76.15 \pm 9.39$    \\
    RF                & $79.23 \pm 9.76$    \\
    XGBoost           & $76.15 \pm 12.14$    \\
    MLP               & $81.54 \pm 7.84$    \\
    Linear STG        & $74.62 \pm 11.44$    \\
    Nonlinear STG     & $76.15 \pm 13.95$    \\
    INVASE            & $76.92 \pm 12.40$    \\
    L2X               & $78.46 \pm 8.28$    \\
    TabNet            & $64.62 \pm 12.02$    \\
    REAL-x            & $75.38 \pm 12.78$    \\
    LSPIN             & $71.54 \pm 6.92$    \\
    LLSPIN            & $83.85 \pm 5.38$    \\
    \bf{ManiFeSt}     & $\boldsymbol{85.23 \pm 8.60}$    \\
    \bottomrule
  \end{tabular}
\end{table}

\subsubsection{Toy example: limitations of ManiFeSt}

ManiFeSt considers multivariate associations rather than univariate properties. In some scenarios, this might lead to the selection of irrelevant features or the misselection of relevant features.

We demonstrate this limitation using a toy example.
We simulate data consisting of $d=20$ binary features and $N=500$ instances. Each feature is sampled from a Bernoulli distribution. We consider two cases. In the first case, each instance is associated with a label that is equal to the first feature $f_1$, where $f_i$ denotes the $i$th feature. Accordingly, only $f_1$ is a relevant for the classification of the label. In the second case, we set $f_5=0$ to be a fixed constant, while the label is still determined by $f_1$.

In Fig. \ref{fig:app_limitations}, we present the normalized feature score obtained by ManiFeSt, averaged over 50 Monte-Carlo iterations of data generation. The green circles denote the average score, and the red dashes indicate the standard deviation. 

In Fig. \ref{fig:app_limitations}(a), we show the results for the first case. We see that ManiFeSt does not identify $f_1$, since the associations of this feature to the other features are not distinct. More specifically, the differences $f_1-f_j$ for every $j$ have the same statistics over samples from the two classes. In Fig. \ref{fig:app_limitations}(b), we show the results for the second case. We see that the relevant feature $f_1$ is captured, yet the irrelevant feature $f_5$ is also detected. This is due to the fact that now the association of $f_1$ with $f_5$ becomes distinct between the two classes. This result implies that ManiFeSt may identify features without any discriminative capabilities through their associations to other relevant and discriminative features. 

Fig. \ref{fig:colon}(b) suggests that the combination of a classical univariate criterion with the multivariate ManiFeSt score has potential to mitigate this limitation. In future work, we will further investigate the simultaneous utilization of univariate and multivariate properties.

\begin{figure}
	\begin{center}
		\includegraphics[width=1\textwidth]{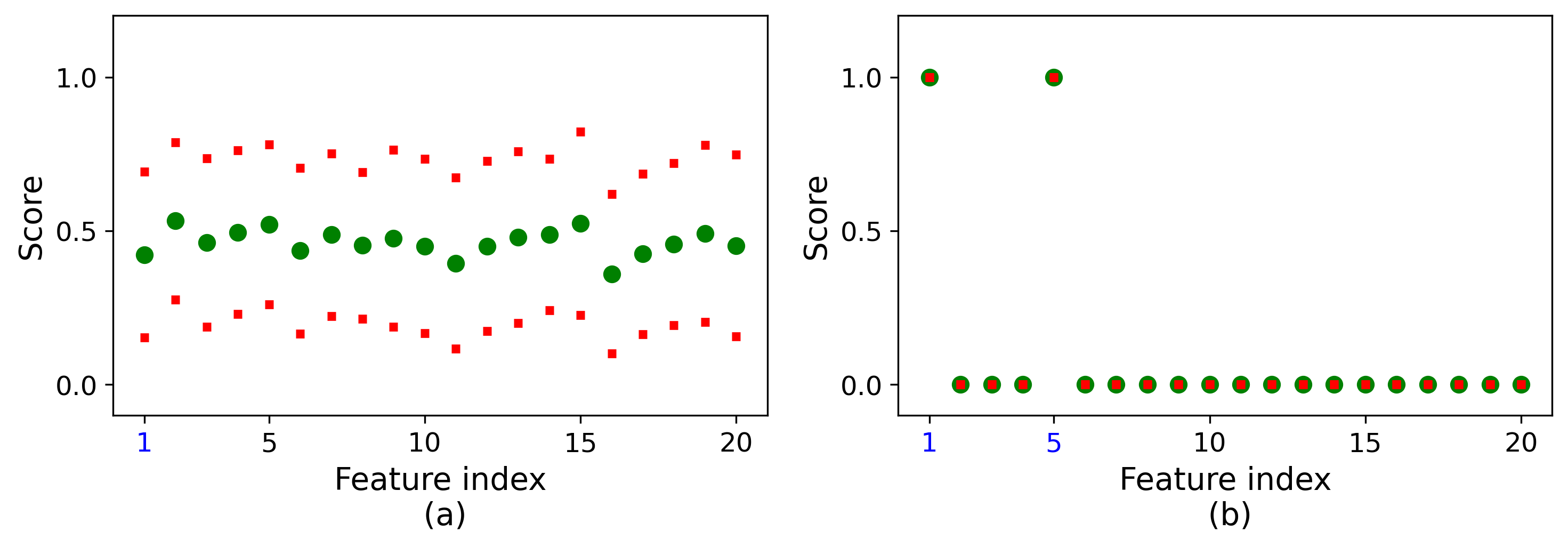}
		\caption{ManiFeSt score on the toy example, averaged over 50 Monte-Carlo iterations of data generation. The green circles denote the average score, and the red dashes indicate the standard deviation. (a) The first case. (b) The second case ($f_5\equiv0$).} 
		\label{fig:app_limitations}
	\end{center}
\end{figure}


\subsubsection{Gisette and Prostate cancer datasets: additional results}

We now test two additional datasets: Gisette and Prostate cancer.

The Gisette dataset \cite{guyon2008feature} is a synthetic dataset from the NIPS 2003 feature selection challenge. The dataset contains 7000 samples consisting of 5000 features, of which only 2500 are relevant features. The classification problem aims to discriminate between the digits 4 and 9, which were mapped into a high dimensional feature space. For more details, see \cite{guyon2003design}.

The Prostate cancer dataset \cite{singh2002gene} consists of the expression levels of 5966 genes (features) and 102 samples, of which 50 are normal and 52 are tumor samples.

The data is divided into train and test sets with a 10-fold cross-validation. 
For evaluation, an SVM classifier is optimized on the selected features subset. 

In Table \ref{tab:more_restuls}, we report the maximum accuracy and STD obtained by selecting any number of features. We note that the baseline test accuracy, when using all the features, is $98.04\%$ and $90.36\%$ for Gisette and Prostate cancer, respectively. 

To provide a clearer picture of these results, Fig.~\ref{fig:gisette_prostae} presents the test accuracy as a function of the number of features used for ManiFeSt and ReliefF, where Fig.~\ref{fig:gisette_prostae}(a) is for Gisette and Fig.~\ref{fig:gisette_prostae}(b) is for the Prostate cancer dataset.

We see that ManiFeSt is on par with the competing methods. Specifically, we see that it achieves high accuracy already by selecting a small number of features. 
In comparison to the applications shown in the paper, these datasets consist of a larger number of features, demonstrating ManiFeSt's capability to perform well in such scenarios as well.

\begin{table}
  \caption{Results on Gisette and Prostate Cancer Datasets}
  \label{tab:more_restuls}
  \centering
  \begin{tabular}{l|l|l|}
    \toprule
    
    {} & \multicolumn{1}{c|}{Gisette}    & \multicolumn{1}{|c|}{Prostate Cancer}               \\
    \cmidrule(r){2-3}
    Method          &Accuracy $\pm$ STD &Accuracy $\pm$ STD \\
    \midrule
    Gini Index         & $98.59\pm 0.38$ & $94.36\pm 8.42$    \\
    ANOVA              & $98.66\pm 0.47$ & $94.36\pm 8.42$    \\
    Pearson            & $98.66\pm 0.47$ & $94.36\pm 8.42$    \\
    T-test             & $98.66\pm 0.47$ & $94.36\pm 8.42$    \\
    Fisher             & $98.66\pm 0.47$ & $94.36\pm 8.42$    \\
    IG                 & $98.59\pm 0.44$ & $93.36\pm 8.28$    \\
    Laplacian          & $98.61\pm 0.4$ & $94.36\pm 8.42$    \\
    ReliefF            & $98.61\pm 0.44$ & $94.36\pm 8.42$    \\
    \bf{ManiFeSt}      & $98.64\pm 0.38$ & $95.27\pm 6.24$    \\
    \bottomrule
  \end{tabular}
\end{table}

\begin{figure}
	\begin{center}
		\includegraphics[width=1\textwidth]{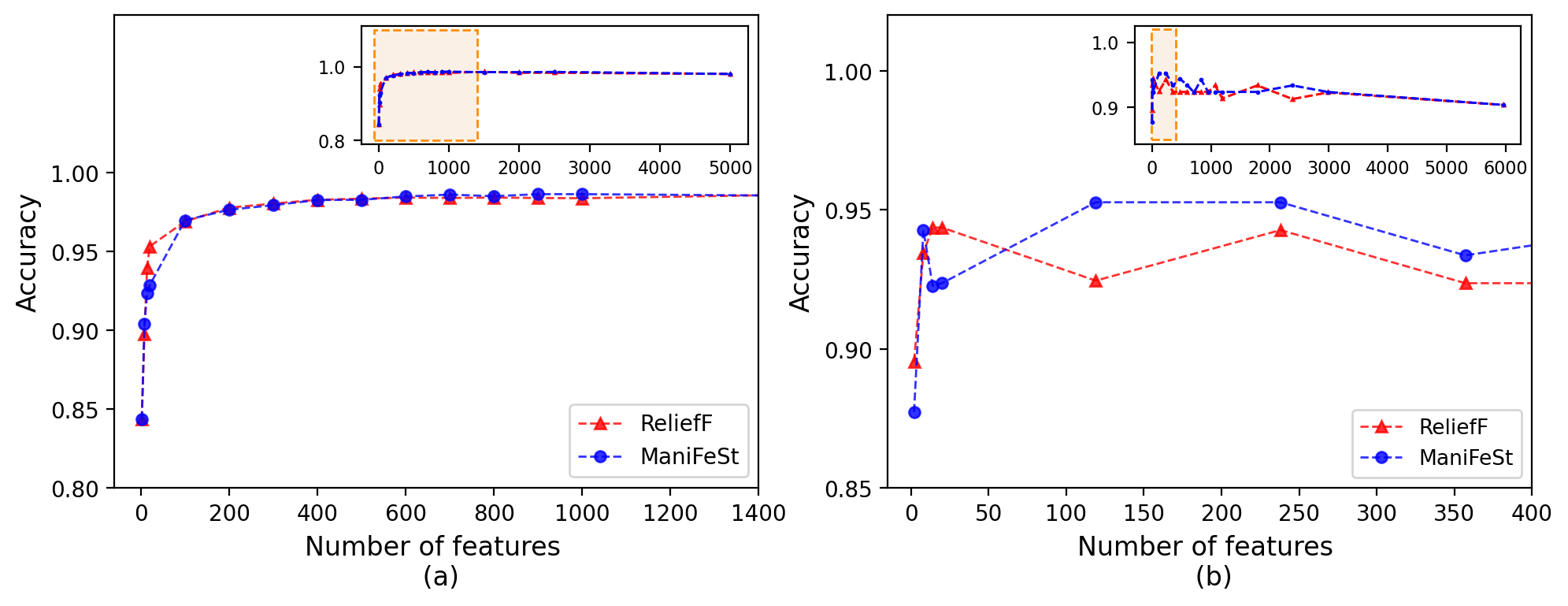}
		\caption{Test accuracy as a function of the number of features used for ManiFeSt and ReliefF. (a) Gisette. (b) Prostate cancer.}
		\label{fig:gisette_prostae}
	\end{center}
\end{figure}

\section{Additional Theoretical Foundation}
\label{app:theoretical}

Here, we provide additional interpretation for the ManiFeSt score, given in \eqref{eq:score}, which was derived based on (purely) geometric considerations. 
We show that high absolute values in eigenvectors of $\boldsymbol{D}$ that correspond to large eigenvalues (in absolute value), represent features with significantly different associations between the two classes.
Due to the eigenvalue weighting in the ManiFeSt score, this verifies that discriminative features will get high scores.

\begin{prop}\label{app:prop:evD_bound_identical}
Assume that $\phi$ is a shared eigenvector of $\boldsymbol{K}_1$ and $\boldsymbol{K}_2$ with respective eigenvalues $\lambda^{(\boldsymbol{K}_1)}$ and $\lambda^{(\boldsymbol{K}_2)}$.
Then $\phi$ is an eigenvector of $\boldsymbol{D}$ with a corresponding eigenvalue $\lambda^{(\boldsymbol{D})}=\sqrt{\lambda^{(\boldsymbol{K}_1)}\lambda^{(\boldsymbol{K}_2)}}\left(\log\lambda^{(\boldsymbol{K}_1)} - \log\lambda^{(\boldsymbol{K}_2)}\right)$ that satisfies:
\begin{equation}
    \left| \lambda^{(\boldsymbol{D})} \right| \le 2\sum _{i,j=1}^d \left| e^{-\|x_i^{(1)} - x_j^{(1)}\|^2/2\sigma^2} - e^{-\|x_i^{(2)} - x_j^{(2)}\|^2/2\sigma^2}\right| |\phi(i)| |\phi(j)|\label{app:eq:evD_bound}
\end{equation}
where $x_i^{(\ell)}$ and $x_j^{(\ell)}$ are vectors containing the values of features $i$ and $j$, respectively, from all the samples in class $\ell$.
\end{prop}
This bound implies that if $\lambda^{(\boldsymbol{D})}$ is large, there must exist pairs of features $i_0$ and $j_0$ that significantly contribute to the sum in the right hand side by satisfying: (i) $|\phi(i_0)|$ and $|\phi(j_0)|$ are large, and (ii) $\left( e^{-\|x_i^{(1)} - x_j^{(1)}\|^2/2\sigma^2} - e^{-\|x_i^{(2)} - x_j^{(2)}\|^2/2\sigma^2}\right)$ is large, implying on a significant difference of the feature associations between the two classes.
In other words, this derivation indicates that a feature $i$ that is discriminative in a multivariate sense, i.e., a feature whose associations with other features are significantly different between the two classes, is represented by a high value $|\phi(i)|$ in eigenvectors that correspond to large eigenvalues, $|\lambda^{(\boldsymbol{D})}|$.
Observing the expression in \eqref{eq:score}, it is evident that such features $i$ with high values of $|\phi(i)|$ in eigenvectors that correspond to eigenvalues with large absolute values $|\lambda^{(\boldsymbol{D})}|$ are assigned with high values of the ManiFeSt score.

\begin{proof}
First, the claim that $\phi$ is also an eigenvector of $\boldsymbol{D}$ with its corresponding eigenvalue, given by $\lambda^{(\boldsymbol{D})}=\sqrt{\lambda^{(\boldsymbol{K}_1)}\lambda^{(\boldsymbol{K}_2)}}\left(\log\lambda^{(\boldsymbol{K}_1)} - \log\lambda^{(\boldsymbol{K}_2)}\right)$, is proved in \cite[Theorem 2]{shnitzer2022spatiotemporal}.
Second, to prove the bound for $\lambda^{(\boldsymbol{D})}$, we start from the definition of the eigenvalue decomposition, for $\ell=\{1,2\}$: 
\begin{equation}
    \lambda^{(\boldsymbol{K}_\ell)} = \phi^T \boldsymbol{K}_\ell \phi = \sum _{i,j=1}^d e^{-\|x_i^{(\ell)} - x_j^{(\ell)}\|^2/2\sigma^2} \phi(i) \phi(j)\label{app:eq:eig_def}
\end{equation}
The difference between the eigenvalues $\lambda^{(\boldsymbol{K}_1)}$ and $\lambda^{(\boldsymbol{K}_2)}$ is then given by:
\begin{equation}
    \left|\lambda^{(\boldsymbol{K}_1)}-\lambda^{(\boldsymbol{K}_2)}\right| = \left| \sum _{i,j=1}^d \left( e^{-\|x_i^{(1)} - x_j^{(1)}\|^2/2\sigma^2} - e^{-\|x_i^{(2)} - x_j^{(2)}\|^2/2\sigma^2}\right) \phi(i) \phi(j) \right|\label{app:eq:ediff}
\end{equation}

Assume w.l.o.g that $\lambda^{(\boldsymbol{K}_1)} > \lambda^{(\boldsymbol{K}_2)}$. Then, we have
\begin{equation}
    |\log \lambda^{(\boldsymbol{K}_1)} - \log \lambda^{(\boldsymbol{K}_2)}| \le 2\frac{\left| \sqrt{\lambda^{(\boldsymbol{K}_1)}} - \sqrt{\lambda^{(\boldsymbol{K}_2)}}\right|}{\sqrt{\lambda^{(\boldsymbol{K}_2)}}}\nonumber
\end{equation}
since $\lambda^{(\boldsymbol{K}_1)},\lambda^{(\boldsymbol{K}_2)}>0$, $\log(x)-\log(y)=\log(x/y)$ and $0\leq\log(x)\leq 2(\sqrt{x}-1)$ for $x\geq 1$. 
Multiplying both sides by $\sqrt{\lambda^{(\boldsymbol{K}_1)}\lambda^{(\boldsymbol{K}_2)}}$ gives the following inequality
\begin{equation}
    \left|\sqrt{\lambda^{(\boldsymbol{K}_1)}\lambda^{(\boldsymbol{K}_2)}}\left(\log \lambda^{(\boldsymbol{K}_1)} - \log \lambda^{(\boldsymbol{K}_2)}\right)\right| \le 2\left| \lambda^{(\boldsymbol{K}_1)} - \sqrt{\lambda^{(\boldsymbol{K}_1)}\lambda^{(\boldsymbol{K}_2)}}\right| \le 2\left| \lambda^{(\boldsymbol{K}_1)} - \lambda^{(\boldsymbol{K}_2)}\right|\label{app:eq:D_eig}
\end{equation}
where the last transition is due to $\sqrt{\lambda^{(\boldsymbol{K}_1)}}>\sqrt{\lambda^{(\boldsymbol{K}_2)}}$, and the left hand side of this equation is equal to $\lambda^{(\boldsymbol{D})}$.
Combining \eqref{app:eq:ediff} and \eqref{app:eq:D_eig} concludes the proof, leading to the following upper bound of the absolute value of $\lambda^{(\boldsymbol{D})}$:
\begin{equation}
    \left| \lambda^{(\boldsymbol{D})} \right| \le 2\sum _{i,j=1}^d \left| e^{-\|x_i^{(1)} - x_j^{(1)}\|^2/2\sigma^2} - e^{-\|x_i^{(2)} - x_j^{(2)}\|^2/2\sigma^2}\right| |\phi(i)| |\phi(j)|
\end{equation}
\end{proof}

Note that a similar derivation can be done for eigenvectors of $\boldsymbol{K}_1$ and $\boldsymbol{K}_2$ that are only approximately similar (not identically shared), as stated by the following proposition. 
\begin{prop}
Let $\phi$ denote an eigenvector of $\boldsymbol{K}_1$ with eigenvalue $\lambda^{(\boldsymbol{K}_1)}$ and $\phi^{(2)}$ denote an eigenvector of $\boldsymbol{K}_2$ with eigenvalue $\lambda^{(\boldsymbol{K}_2)}$. 
Assume that $\phi^{(2)}=\phi+\phi_\epsilon$, where $\left\Vert\phi_\epsilon\right\Vert_2<\epsilon$ for some small $\epsilon>0$.
Then $\phi$ and $\lambda^{(\boldsymbol{D})}=\sqrt{\lambda^{(\boldsymbol{K}_1)}\lambda^{(\boldsymbol{K}_2)}}\left(\log\lambda^{(\boldsymbol{K}_1)} - \log\lambda^{(\boldsymbol{K}_2)}\right)$ are an approximate eigen-pair of $\boldsymbol{D}$, such that $\lambda^{(\boldsymbol{D})}$ satisfies:
\begin{equation}
    \left| \lambda^{(\boldsymbol{D})} \right| \le 2\sum _{i,j=1}^d \left| e^{-\|x_i^{(1)} - x_j^{(1)}\|^2/2\sigma^2} - e^{-\|x_i^{(2)} - x_j^{(2)}\|^2/2\sigma^2}\right| |\phi(i)| |\phi(j)| + 2a\epsilon^2\label{app:eq:evD_approx_bound}
\end{equation}
where $a=\max_i\sum_j K_2(i,j)$.
\end{prop}

\begin{proof}
A proof showing that $\phi$ and $\lambda^{(\boldsymbol{D})}$ are an approximate eigen-pair of $\boldsymbol{D}$ can be found in \cite[Theorem 4]{shnitzer2022spatiotemporal}.
For the bound on $\lambda^{(\boldsymbol{D})}$, note that equation \eqref{app:eq:eig_def} holds for $\boldsymbol{K}_1$ with no change, whereas for $\boldsymbol{K}_2$ we have:
\begin{eqnarray}
    \lambda^{(\boldsymbol{K}_2)} & = & \left(\phi^{(2)}\right)^T \boldsymbol{K}_2 \phi^{(2)} = \phi^T \boldsymbol{K}_2 \phi + \phi_\epsilon^T \boldsymbol{K}_2 \phi_\epsilon\nonumber\\
    & = & \sum _{i,j=1}^d e^{-\|x_i^{(2)} - x_j^{(2)}\|^2/2\sigma^2} \phi(i) \phi(j) + \phi_\epsilon^T \boldsymbol{K}_2 \phi_\epsilon
\end{eqnarray}
The difference between the eigenvalues can then be bounded by:
\begin{eqnarray}
    \left|\lambda^{(\boldsymbol{K}_1)}-\lambda^{(\boldsymbol{K}_2)}\right| & = & \left| \sum _{i,j=1}^d \left( e^{-\|x_i^{(1)} - x_j^{(1)}\|^2/2\sigma^2} - e^{-\|x_i^{(2)} - x_j^{(2)}\|^2/2\sigma^2}\right) \phi(i) \phi(j) - \phi_\epsilon^T \boldsymbol{K}_2 \phi_\epsilon\right|\nonumber\\
    & \leq & \left| \sum _{i,j=1}^d \left( e^{-\|x_i^{(1)} - x_j^{(1)}\|^2/2\sigma^2} - e^{-\|x_i^{(2)} - x_j^{(2)}\|^2/2\sigma^2}\right) \phi(i) \phi(j)\right| + \lambda^{(\boldsymbol{K}_2)}_{\max}\epsilon^2\nonumber\\
    & \leq & \left| \sum _{i,j=1}^d \left( e^{-\|x_i^{(1)} - x_j^{(1)}\|^2/2\sigma^2} - e^{-\|x_i^{(2)} - x_j^{(2)}\|^2/2\sigma^2}\right) \phi(i) \phi(j)\right| + a\epsilon^2
\end{eqnarray}
where $\lambda^{(2)}_{\max}$ is the maximal eigenvalue of $\boldsymbol{K}_2$, which is bounded by $a=\max_i\sum_j K_2(i,j)$ (maximal row sum of $\boldsymbol{K}_2$) from the Perron-Frobenius theorem. 
Note that since $\boldsymbol{K}_2$ is positive semi-definite with $1$s on its diagonal, its maximal eigenvalue can also be crudely bounded by $d$, the dimension of the feature space.
The rest of the derivation for the approximate case follows from \eqref{app:eq:D_eig}, resulting in:
\begin{equation}
    \left|\lambda^{(\boldsymbol{D})}\right| \le 2\sum _{i,j=1}^d \left| e^{-\|x_i^{(1)} - x_j^{(1)}\|^2/2\sigma^2} - e^{-\|x_i^{(2)} - x_j^{(2)}\|^2/2\sigma^2}\right| |\phi(i)| |\phi(j)| + 2a\epsilon^2
\end{equation}
\end{proof}

\end{document}